\documentclass[lettersize,journal]{IEEEtran}
\usepackage{amsmath,amsfonts}
\usepackage{array}
\usepackage{tcolorbox}   
\usepackage{siunitx}
\usepackage{subfigure}
\usepackage{subcaption}
\usepackage[dvipsnames]{xcolor} 
\usepackage{textcomp}
\usepackage{stfloats}
\usepackage{url}
\usepackage{bm}
\usepackage{algorithm}
\usepackage{algpseudocode}
\usepackage{verbatim}
\usepackage{booktabs}
\usepackage{graphicx}
\usepackage{multirow} 
\usepackage{color}
\usepackage{diagbox}
\usepackage{amsthm}
\usepackage{hyperref}
\usepackage{balance}

\newcommand{\mytheoremname}{\bfseries Theorem}
\newcommand{\mylemmaname}{\bfseries Lemma}
\newcommand{\mypropositionname}{\bfseries Proposition}
\newtheorem{Theorem}{\mytheoremname} 
\newtheorem{Lemma}[Theorem]{\mylemmaname} 
\newtheorem{Proposition}[Theorem]{\mypropositionname} 
\tcolorboxenvironment{Theorem}{
  colback=Blue!0.5!White,  
  colframe=white, 
  boxrule=0.1pt,         
}
\tcolorboxenvironment{Lemma}{
  colback=Blue!0.5!White,  
  colframe=white, 
  boxrule=0.1pt,         
}
\def\BibTeX{{\rm B\kern-.05em{\sc i\kern-.025em b}\kern-.08em
    T\kern-.1667em\lower.7ex\hbox{E}\kern-.125emX}}

\begin{document}
\title{ Tackling Over-smoothing on Hypergraphs: A Ricci Flow-guided Neural Diffusion Approach}
\author{Mengyao Zhou, Zhiheng Zhou, Xiao Han, Xingqin Qi, Guanghui Wang, Guiying Yan
\thanks{ This work was supported by the National Natural Science Foundation of China (No. 12231018).(Mengyao Zhou and Zhiheng Zhou contributed equally to this work.)
(Corresponding authors:  Guiying Yan.)}
\thanks{Mengyao Zhou and Guiying Yan are with the Academy of Mathematics and Systems Science, Chinese Academy of Sciences and also with the University of Chinese Academy of Sciences, Beijing 100190, China (e-mail: zhoumengyao@amss.ac.cn; yangy@amss.ac.cn).}
\thanks{Zhiheng Zhou and Xingqin Qi a with the School of Mathematics and Statistics, Shandong University, Weihai, Shandong 264209, China (e-mail: zhouzhiheng@amss.ac.cn;qixingqin@sdu.edu.cn).}
\thanks{Xiao Han is with School of Artificial Intelligence, Beihang University, Beijing 100191,China (e-mail: hx2210@buaa.edu.cn).}
\thanks{Guanghui Wang is with the School of Mathematics, Shandong University, Jinan, Shandong 250100,China (e-mail: 
ghwang@sdu.edu.cn).}
}

\markboth{}%
{How to Use the IEEEtran \LaTeX \ Templates}\maketitle
This work has been submitted to the IEEE for possible publication. Copyright may be transferred without notice, after which this version may no longer be accessible.

\begin{abstract}
Hypergraph neural networks (HGNNs) have demonstrated strong capabilities in modeling complex higher-order relationships. However, existing HGNNs often suffer from over-smoothing as  the number of layers increases and  lack effective control over message passing among nodes.  Inspired by the theory of Ricci flow in differential geometry, we theoretically establish that introducing discrete Ricci flow into hypergraph structures can effectively regulate node feature evolution and thereby alleviate over-smoothing. Building on this insight, we propose Ricci Flow-guided  Hypergraph Neural Diffusion(RFHND), a novel message passing paradigm for hypergraphs guided by discrete Ricci flow. Specifically, RFHND is based on a PDE system that describes the continuous evolution of node features on hypergraphs and adaptively regulates the rate  of information diffusion at the geometric level, preventing feature homogenization and producing high-quality node representations. Experimental results show that RFHND significantly outperforms existing methods across multiple benchmark datasets and demonstrates strong robustness, while also effectively mitigating over-smoothing. 
\end{abstract}

\begin{IEEEkeywords}
Hypergraph Neural Networks, Over-Smoothing, Ricci Flow, Differential Equation
\end{IEEEkeywords}

\section{Introduction}
In recent years, hypergraphs have attracted widespread attention as an important tool for modeling complex systems~\cite{zhou2006learning,antelmi2023survey}. Unlike traditional graphs, which can only represent pairwise relationships between nodes, hypergraphs can connect multiple nodes through hyperedges, naturally capturing higher-order relations~\cite{berge1984hypergraphs}. This capability gives hypergraphs unique expressive power in many domains such as social networks~\cite{zlatic2009hypergraph,zhu2018social}, recommendation systems~\cite{xia2022self,wang2020next}, and biological networks~\cite{feng2021hypergraph,klamt2009hypergraphs}. To fully exploit the information from hypergraph data, researchers have proposed various hypergraph neural network methods~\cite{kim2024survey}, achieving significant progress in tasks such as node classification~\cite{wu2022hypergraph}, link prediction~\cite{yadati2020nhp,li2013link}, and representation learning~\cite{antelmi2023survey}.

Although existing hypergraph neural networks provide powerful tools for hypergraph learning, the problem of over-smoothing remains a key challenge~\cite{li2025deep,chen2022preventing}. As the network deepens, node features tend to converge and become increasingly indistinguishable, which degrades the model's performance~\cite{lin2022deephgnn}.  Current mitigation strategies fall into two main paradigms. The first paradigm modifies the network architecture by incorporating mechanisms such as residual connections to preserve initial node features~\cite{huang2021unignn,chien2021you}. The second paradigm refines the aggregation step, often using attention mechanism to selectively weigh nodes and hyperedges, thereby preventing feature homogenization~\cite{bai2021hypergraph,chen2020hypergraph}. However, while existing methods mitigate over-smoothing to some extent, most of them are essentially operator-level fixes  and lack rigorous theoretical guarantees, which limits their efficacy in complex scenarios. 

Inspired by the work~\cite{nguyen2023revisiting},  which establishes a direct link between geometric curvature and over-smoothing, we identify the feature convergence in hypergraph neural networks as an unconstrained geometric evolution process, analogous to  heat diffusion. This perspective motivates us to seek a mechanism from differential geometry that can intrinsically control such diffusion dynamics. Ricci flow serves precisely this purpose~\cite{chow2004ricci},  as it describes how the metric tensor evolves according to the local curvature. Recently, its generalization, discrete Ricci flow, has proven effective for tasks on graphs~\cite{yang2009generalized}. However, its potential for hypergraphs remains unexplored. In this work, we generalize this evolution to hypergraph and theoretically demonstrate that the evolution constrains the diffusion of node features via local curvature, which in turn effectively mitigates over-smoothing.

Motivated by the preceding theoretical insights, we propose Ricci Flow-guided Hypergraph
Neural Diffusio(RFHND). Grounded in partial differential equations, RFHND models node feature evolution as a continuous dynamical process and adaptively regulates information diffusion based on local curvature, effectively preventing over-smoothing. The core idea of RFHND is shown in Fig. \ref{fig:ricci}. Specifically, we assign weights to hyperedges proportional to the feature similarity among their nodes, enabling an adaptive flow of updates that mitigates feature homogenization. This design preserves node distinctions while enabling efficient feature fusion, enhancing model expressiveness and stability. Finally,
the efficacy of our algorithm is validated through a comprehensive series of experiments.
The main contributions of this work are as follows:
\begin{itemize}
    \item \textbf{Theoretical foundation:} We introduce discrete Ricci flow to the hypergraph learning domain, theoretically proving that it mitigates over-smoothing by adaptively controlling feature diffusion via local curvature.
    \item \textbf{Methodological contribution:}  We propose RFHND, a discrete Ricci flow–guided hypergraph neural diffusion method that transforms traditional  message passing into an adaptive, curvature-guided diffusion process, thereby effectively preventing feature homogenization.
    \item \textbf{Empirical validation: }Experiments on several hypergraph datasets demonstrate that RFHND achieves superior node classification accuracy while exhibiting improved stability, robustness, and a significant reduction in over-smoothing compared to existing methods.
\end{itemize}

\begin{figure*}[htb]
\centering  
\includegraphics[width=15cm,height =7cm]{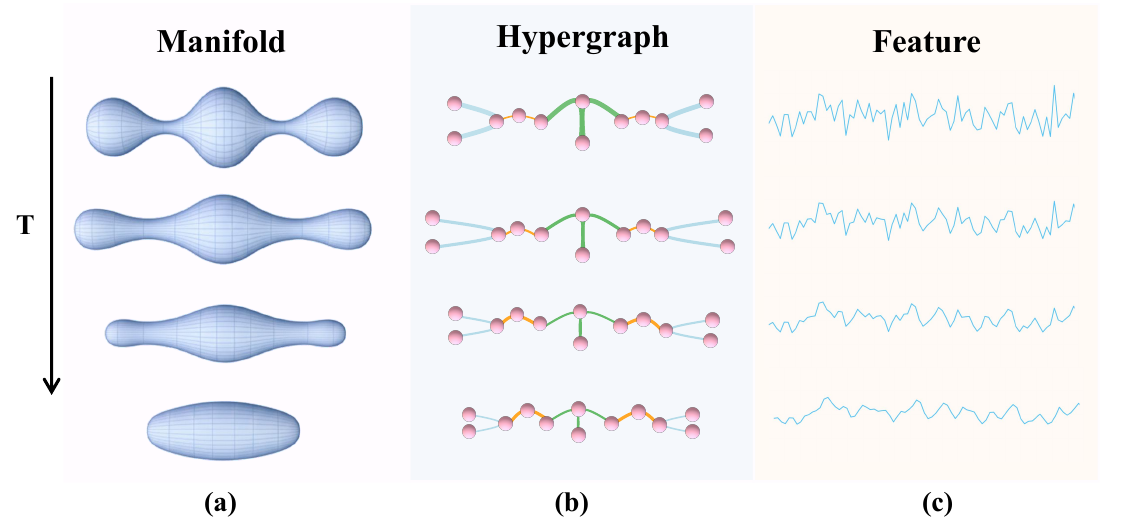}
 \caption{Ricci flow-guided evolution. (a) Evolution of the manifold geometry under Ricci flow. (b) Evolution of hyperedge weight guided by Ricci flow. (c) Evolution of node feature distributions driven by the updated hyperedge weights.}
\label{fig:ricci}
\end{figure*}

\section{Related Work}
In this section, we review literature closely related to our work, focusing on two key areas:  Hypergraph Neural Networks(HGNNs) and Over-smoothing in HGNNs.
\subsection{Hypergraph Neural Networks}
Hypergraph neural networks have emerged as a powerful generalization of graph neural networks (GNNs), specifically designed to capture complex high-order relationships among entities beyond pairwise connections. The  development of this field was initiated by the pioneering work of the foundational HGNN model~\cite{feng2019hypergraph}. This model utilizes a spectral convolution framework on hypergraphs to aggregate information from nodes connected with the same hyperedge.

Following this, a wave of message-passing-based models, such as HNHN~\cite{dong2020hnhn}, HyperGCN~\cite{yadati2019hypergcn}, HyperSAGE~\cite{arya2020hypersage}, and UniGNN~\cite{huang2021unignn} was proposed to enhance the expressive power of HGNNs. These methods extend aggregation and propagation mechanisms to accommodate heterogeneous and weighted hyperedges, leading to significant improvements in tasks such as node classification. Meanwhile, architectures like AllDeepSets~\cite{chien2021you} and AllSetTransformer~\cite{chien2021you}  abandoned spectral assumptions entirely. They instead employ permutation-invariant set functions over hyperedges to facilitate more flexible, set-level reasoning.

Recent HGNN advances include more sophisticated modeling perspectives. For instance, dynamic system models like HDS~\cite{yan2024hypergraph} use ordinary differential equations (ODEs) to improve the stability and control of the learning process. Additionally, models such as KHGNN(employs a nested convolution module named HyperGINE)~\cite{xie2025k} are designed to capture long-range dependencies by extracting features from nodes, hyperedges, and the intermediate paths between them.

\subsection{Over-smoothing in HGNNs}
In deep hypergraph networks, node representations tend to become identical, which severely degrades model performance. Current mitigation strategies, designed to combat this issue, can be categorized into two main paradigms.

\textbf{Architectural Modifications.} This line of work focuses on altering the network's architecture to preserve information from earlier layers. Inspired by the effectiveness of residual connections in both CNNs and GNNs, models like UniGCNII~\cite{huang2021unignn}  and Deep-HGCN~\cite{chien2021you} integrate skip connections to carry initial node features forward, enabling the construction of much deeper and more effective HGNNs.  FrameHGNN~\cite{li2025deep} introduces a framelet-based hypergraph convolution framework that combines low and high-pass filters with proven deep GNN techniques like residual and identity mapping, thereby  maintaining discriminative signals in deep layers.

\textbf{Refining the Aggregation Operator.} This paradigm aims to make the feature aggregation process more discerning. Instead of uniform aggregation, attention mechanisms are employed to assign different weights to nodes within a hyperedge or to different hyperedges connected to a node. For example, Hypergraph Attention Networks~\cite{bai2021hypergraph,chen2020hypergraph}  learn dynamic weights for both nodes and hyperedges,  thereby effectively mitigating the issue of indiscriminate feature mixing. ED-HNN~\cite{wang2022equivariant}   suggests that employing equivariant operators to distribute distinct messages across nodes helps preserve feature diversity (or node distinguishability), thereby effectively preventing the over-smoothing problem.

\textbf{Limitations of Existing Methods.} 
While existing methods mitigate over-smoothing to some extent, most of them rely on local framework modifications and lack rigorous theoretical guarantees, which limits their efficacy in complex scenarios. In contrast, we introduce discrete Ricci flow to globally regulate the evolution of node features, thereby effectively alleviating the over-smoothing issue.

\section{Preliminary}
\subsection{Notations} 
 Let $\mathcal{H} = (\mathcal{V} , \mathcal{E})$ denote a  hypergraph, where $\mathcal{V}$ is the vertex set containing $n$ unique vertices and $\mathcal{E}$ is the edge set containing $m$ hyperedges.  The hypergraph can be represented by an incidence matrix $H\in R^{n\times m}$  where $H_{ij} = 1$ if the vertex ${v}_{i}\in \mathcal{V}$ is contained in the hyperedge $  {e_j} \in \mathcal{E}$, otherwise 0.  Let $\mathbf{F}=[\mathbf{f}_{1},\mathbf{f}_{2},... \mathbf{f}_{n}]^{T} $ denote the node feature matrix, and $f_{i}$ is associated with the node $v_{i}$. Similarly, let $\mathbf{X}(t)=[\mathbf{x}_{1}(t),\mathbf{x}_{2}(t),...\mathbf{x}_{n}(t)]$ denote the node representation evolved to time $t$, and $\mathbf{X}(0)= Linear(\mathbf{F}) $. Each hyperedge $e_j \in \mathcal{E}$ is assigned a positive  time-varying weight $w_{e_j}(t)$, all the weights formulate a diagonal matrix $W(t) \in R^{m\times m}$.
 The vertex and edge degree of hypergraph can be expressed  as $d_{i} = \sum_{j=1}^{m} H_{ij}$ and $d_{j} = \sum_{i=1}^{n} H_{ij}$, all the degrees formulate the diagonal matrix $D_{v}$ and $D_{e}$. 
 \subsection{Hyperedge Curvature on Weighted Hypergraph}
 \label{B}
Let $\mathcal{H}$ denote a weighted hypergraph. 
Each hyperedge $e \in \mathcal{E}$ is assigned a weight $w_e$, and a curvature value $\kappa_e$ .  Specifically, $\kappa_e$ reflects the closeness of the nodes within the hyperedge.

Several discrete notions of curvature have been extended from graphs to hypergraphs, including  Forman–Ricci curvature $\kappa_e^{FR}$~\cite{leal2021forman},  Ollivier–Ricci curvature $\kappa_e^{OR}$ ~\cite{coupette2022ollivier} and Resistance curvature $\kappa_e^{RC}$~\cite{eidi2020edge}. Notably, our method is not restricted  to a specific choice of curvature measure.
For the precise mathematical formulations of these curvature measures, please refer to  the supplementary material.


\subsection{Dirichlet Energy}
The Dirichlet energy of the hypergraph $\mathcal{H} = (\mathcal{V} , \mathcal{E})$ is used to characterize the smoothness of features between nodes~\cite{cai2020note}. The expression is as follows~\cite{chien2021you}:
\begin{equation}
\label{de}
    E(\mathbf{X}(t)) = \frac{1}{2}\sum_{e\in \mathcal{E}} \sum_{i,j \in e}\frac{1}{|e|}\left(\frac{\mathbf{x}_{i}(t)}{\sqrt{d_{i}}}-\frac{\mathbf{x}_{j}(t)}{\sqrt{d_{j}}}\right)^{2}.
\end{equation}
The equation (\ref{de}) shows that a higher similarity between neighboring node features leads to an energy value closer to 0, which reflects the over-smoothing problem.
\subsection{Discrete Ricci Flow(DRF)}  
Ricci flow, originally introduced by Hamilton in differential geometry~\cite{hamilton1982three}, evolves a Riemannian metric $g_{ij}(t)$ according to the partial differential equation:
\begin{equation}
\frac{\partial g_{ij}(t)}{\partial t} = -2 \, \mathrm{Ric}_{ij}(g(t)),
\end{equation}
where $g_{ij}(t)$ denotes the element of the metric tensor $g(t)$ and $\mathrm{Ric}_{ij}$ denotes the element of the Ricci curvature tensor. This process can be regarded as a geometric heat equation, smoothing out irregularities of the manifold metric over time, and has played a fundamental role in the development of modern geometry and topology.

To extend this concept to discrete structures, discrete Ricci flow has been proposed as an analogue of the continuous Ricci flow on graphs and networks~\cite{ollivier2009ricci}. The specific expression is as follows:
\begin{equation}
\frac{\partial w_{ij}(t)}{\partial t} = -\kappa_{ij}(t) w_{ij}(t).
\end{equation}
Here, $w_{ij}(t)$ and $\kappa_{ij}(t)$  represent the weight and the curvature of the edge $(v_{i},v_{j})$ at time $t$.
Discrete Ricci flow has found applications in diverse fields, including community detection~\cite{ni2019community}, network analysis~\cite{ni2018network}, graph embedding~\cite{chen2025graph}, and machine learning.

Building upon these advances, we further generalize discrete Ricci flow  to hypergraphs. The core of this generalization lies in reformulating the notions of edge weight and curvature from pairwise edges to multi-node hyperedges. The specific expression of the hypergraph Ricci flow is as follows:
\begin{equation}
\frac{\partial w_{e}(t)}{\partial t} = -\kappa_{e}(t) w_{e}(t).
\label{ricci}
\end{equation}
Here, $w_{e}(t)$ and $\kappa_{e}(t)$  represent the weight and the curvature of the edge $e$ at time $t$.
By formulating  discrete Ricci flow on hypergraphs, we aim to capture higher-order   structures and provide a new theoretical framework for representation learning in Hypergraph.

\section{Applying DRF on Hypergraph}
\label{Applying}
In this section, we apply discrete Ricci flow to both the hypergraph structure and its node features to construct a dynamic evolutionary system. Similar to the attribute discrete Ricci flow in ~\cite{chen2025graph} , we construct the attribute discrete Ricci flow on the hypergraph. Specifically, We define the hyperedge weights  as a function of the nodes features it connects, i.e., $w_{e}(t)\mathrel{\mathop:}=w_{e}(\mathbf{X}_{e}(t))$, Where $\mathbf{X}_{e}(t)=\left\{\mathbf{x}_{i}(t)|i\in e\right\}$ denotes the set of node features on the edge $e$. Then the  attribute discrete Ricci flow of the hypergraph can be written as follow:
\begin{equation}
\frac{\partial w_{e}(\mathbf{X}_{e}(t))}{\partial t} = -\kappa_{e}(t) w_{e}(\mathbf{X}_{e}(t)).
\label{hdrf}
\end{equation}
As established in \ref{B}, the curvature of a hyperedge is jointly determined by both structural connectivity and its weight. In turn, this curvature governs the dynamic evolution of node attributes according to Equation (\ref{hdrf}). Unlike traditional heat diffusion, which models node representation propagation , Ricci flow can be viewed as heat diffusion of the metric.  From now on, we define $w_{e}(t)$ as follows:
\begin{equation}
\label{weight}
\begin{split}
    &w_{e}(t)\equiv \frac{1}{\alpha_{e} } [\frac{1}{|e|}\sum_{i\in e, j\in e}\frac{cos(\mathbf{x}_{i}(t),\mathbf{x}_{j}(t))}{\sqrt{d_{i}d_{j}}}]+1+\epsilon,
\\
 &\alpha_{e} = \frac{1}{|e|}\sum_{i\in e, j\in e}\frac{1}{\sqrt{d_{i}d_{j}}}.
 \end{split}
\end{equation}
    Here, $w_e(t)$ is a non-negative weight, 
$\epsilon$ is a small positive number. Furthermore, we enforce a unit norm on $\mathbf{x}(t)$ (i.e., $ |\mathbf{x}(t)|\equiv1$, which is crucial for maintaining stability during computation. In this case, $w_{e}(t)$ satisfies: $w_{e}(t)\in[\epsilon,2+\epsilon]$.

\subsection{Dirichlet Energy Bound}
\label{th:de}
By applying the attribute discrete Ricci flow on hypergraph, we can prove that the system's energy value remains bounded when the node feature evolution follows the Ricci flow. This demonstrates that our method can effectively prevent over-smoothing and maintain controlled feature differentiation during propagation. The specific conclusions are as follows:
\begin{Theorem}
\label{thedir}
Consider the attribute discrete Ricci flow with hyperedge weight as equation (\ref{weight}), and $|\mathbf{x}(t)|\equiv1$. If $H$ is a non-regular hypergraph, and $w_{e}(t)$  are monotonic on $[t_{1},t_{2}]$ for all $e$ in $H$, then the average Dirichlet energy within $[t_{1},t_{2}]$ has following bound:
\begin{equation}
\begin{split}
&\rho_{max}\left(\sum_{e\in \mathcal
{E}}\sum_{i, j\in e}\frac{1}{2|e|}\left(\frac{1}{d_i} +\frac{1}{d_j}\right)+\sum_{e\in \mathcal{E}}\alpha_e \right) \\&\geq \mathbb{E}_{t\in[t_{1},t_{2}]}(E(\mathbf{X}(t)))\\&\geq \sum_{e\in \mathcal{E}}\sum_{i, j\in e}\frac{1}{2|e|}\left(\frac{1}{d_i} +\frac{1}{d_j}\right)-\sum_{e\in \mathcal{E}}\alpha_e \textgreater 0,
\end{split}
\end{equation}
Here, $d_i$ denotes the degree of node $i$, $\rho_{max}=\max_{e}\rho_e$ and $\rho_e$ denotes the ratio of the maximum to the minimum value of $w_e(t)$ .
\end{Theorem}
As demonstrated,  Theorem \ref{thedir}  provides crucial theoretical support for addressing the over-smoothing problem in hypergraph neural networks.  The theorem's strictly positive lower bound ensures that the differentiation among nodes will not completely vanish throughout the feature evolution process. This, in turn, preserves meaningful feature distinctions and prevents the excessive assimilation of node representations. The full proof is available in the  supplementary material. 

\subsection{Convergence Analysis}
Following the preceding analysis of the method's effectiveness against over-smoothing, we further investigate its convergence. The specific results of this analysis are as follows:
\begin{Theorem}
\label{shoulian}
Consider the attribute discrete Ricci flow  with hyperedge weight as equation (\ref{weight}). Assume that there exists a constant $L$ such that $|\kappa_e(t_2) - \kappa_e(t_1)| \ge L|w_e(t_2) - w_e(t_1)|$ holds for all arbitrary $t_1, t_2$. Let $|\kappa_{e}(0)| > 0$, then for any arbitrarily small positive number $\delta$, it holds that:
\begin{equation}
\min_{t\in[0, +\infty)} \Big \{t \;\Big|\;|\kappa_{e}(t)| = \delta\Big \} \le \frac{1}{L\epsilon - \delta} \ln \left( \frac{2L + \delta}{\delta(2+\epsilon)} \right).
\end{equation}
\end{Theorem}
The Theorem \ref{shoulian}  prove that the discrete Ricci flow possesses the property of exponential convergence. This implies that the hyperedge curvature can rapidly approach zero, regardless of the system's initial state. The Theorem \ref{shoulian} thus provide  a theoretical guarantee for the efficiency and stability of our method. The proof of Theorem \ref{shoulian} can be referred to ~\cite{chen2025graph}.
\section{Methodology}
\subsection{Ricci Flow-guided Hypergraph Feature Diffusion}
\label{fangcheng}
As established in section \ref{Applying}, applying attribute discrete Ricci flow to hypergraphs yields several key advantages. The process drives the curvature of hyperedges toward uniformity while facilitating the learning of high-quality, non-smooth node representations. Moreover, the method is theoretically proven to have strong convergence properties.
Inspired by these strengths, in this section we design a new hypergraph feature diffusion architecture. 
By expanding Equation (\ref{hdrf}) using the chain rule, we obtain:
\begin{equation}
\label{chain}
\sum_{i\in e}\left\langle \frac{\partial w_e}{\partial \mathbf{x}_i}, \frac{\partial \mathbf{x}_i(t)}{\partial t} \right\rangle  = -\kappa_{e}(t) w_e(t).
\end{equation}
The left side of Equation (\ref{chain})  contains $|e|$ parts. To weigh these  parts, we introduce several scaling functions:
\begin{equation}
\left\langle \frac{\partial w_e}{\partial \mathbf{x}_j}, \frac{\partial \mathbf{x}_j(t)}{\partial t} \right\rangle = \lambda^{j}(\mathbf{x}_1,... \mathbf{x}_{|e|}) \left\langle \frac{\partial w_e}{\partial \mathbf{x}_i}, \frac{\partial \mathbf{x}_i(t)}{\partial t} \right\rangle.
\end{equation}
Then, We focus  on one side:
\begin{equation}
\label{lamda}
\left\langle \frac{\partial w_e}{\partial \mathbf{x}_i}, \frac{\partial \mathbf{x}_i(t)}{\partial t} \right\rangle = -\frac{\kappa_{e}(t) w_e(t)}{1 + \sum_{j\in e, j\neq i}\lambda^{j}(\mathbf{x}_1,... \mathbf{x}_{|e|})}.
\end{equation}

 Under the the constraint of Equation (\ref{lamda}) and $|\mathbf{x}_i(t)| \equiv 1$ , we minimize $\|\frac{\partial \mathbf{x}_i(t)}{\partial t}\|$, which means that $\mathbf{x}_i(t)$ always applies only the slightest change to satisfy the attribute discrete Ricci flow, which guarantees that the evolution of $\mathbf{x}_i(t)$ is numerically stable. On this basis, the resulting optimization objective can be formally expressed as follows:
\begin{equation}
\label{youhua}
\begin{split}
&\min \left\| \frac{\partial \mathbf{x}_i(t)}{\partial t} \right\|, \quad \\&\text{s.t.} \;\; \left\langle \frac{\partial w_e}{\partial \mathbf{x}_i}, \frac{\partial \mathbf{x}_i(t)}{\partial t} \right\rangle = -\frac{\kappa_{e}(t) w_e(t)}{1 + \sum_{j\in e, j\neq i}\lambda^{j}(\mathbf{x}_1,... \mathbf{x}_{|e|})} , \\&|\mathbf{x}_i(t)| = 1.
\end{split}
\end{equation}
\begin{Proposition}
\label{proposition3}
 The optimization objective presented in  Equation (\ref{youhua}) has a closed-form solution  as follows:
\begin{equation}
\label{ricci}
\frac{\partial \mathbf{x}_i(t)}{\partial t} = -\kappa'_{ie}(t) \left[\sum_{j\in e} \frac{\mathbf{x}_j(t) - \cos \left( \mathbf{x}_i(t), \mathbf{x}_j(t) \right) \mathbf{x}_i(t)}{\sqrt{d_id_j}} \right],
\end{equation}
where
  \begin{equation}
  \label{kprim}\kappa'_{ie}(t)=\frac{\mu_{ie}}{(\mathbf{1} -{ (\mathbf{x}_i(t)^Tm_{ie} )}^{2})(\sum_{i,j\in e}\frac{1}{\sqrt{d_id_j}})},\end{equation}\begin{equation}
  \mu_{ie} = -\frac{\kappa_{e}(t) w_e(t)}{1 + \sum_{j\in e, j\neq i}\lambda^{j}(\mathbf{x}_1(t),... \mathbf{x}_{|e|}(t))},\end{equation}
 \begin{equation}
m_{ie}=\frac{1}{\sum_{i,j\in e}\frac{1}{\sqrt{d_id_j}}}\left(\sum_{j\in e}\frac{\mathbf{x}_j(t)}{\sqrt{d_id_j}}\right).\end{equation}
\end{Proposition}
To reduce computational complexity, we select  appropriate scaling functions $\lambda^j(\cdot)$ such that $\kappa'_{ie}(t) = \kappa'_{je}(t)$ holds for all $i, j \in e$. Consequently, Equation (\ref{ricci}) can be simplified to:
\begin{equation}
\label{ricci1}
\frac{\partial \mathbf{x}_i(t)}{\partial t} = -\kappa'_{e}(t) \left[\sum_{j\in e} \frac{\mathbf{x}_j(t) - \cos \left( \mathbf{x}_i(t), \mathbf{x}_j(t) \right) \mathbf{x}_i(t)}{\sqrt{d_id_j}} \right].
\end{equation}
The Equation (\ref{ricci1}) defines the local influence of a single hyperedge e on node $i$, while a node's representation is updated based on its entire neighborhood. Accordingly, we generalize this equation from a single hyperedge to the entire hypergraph by iterating over all hyperedges incident to node $i$ and aggregating the corresponding information. We can rewrite Equation (\ref{ricci1}) as follows:
\begin{equation}
\label{ricci flow}
\frac{\partial \mathbf{x}_i(t)}{\partial t} = \sum_{e:i\in e}-\kappa'_{e}(t) \left[\sum_{j\in e} \frac{\mathbf{x}_j(t) - \cos \left( \mathbf{x}_i(t), \mathbf{x}_j(t) \right) \mathbf{x}_i(t)}{\sqrt{d_id_j}} \right].
\end{equation}
We can rewrite the equation (\ref{ricci flow}) in matrix form as follows:
\begin{equation}
\label{ricci flow matrix}
\frac{\partial \mathbf{X}(t)}{\partial t} = \bigg[ \operatorname{diag}\Big( \big( S(\mathbf{X}(t)) \odot C(\mathbf{X}(t)) \big) \mathbf{1}_{N} \Big) - S(\mathbf{X}(t)) \bigg] \mathbf{X}(t),
\end{equation}
\begin{equation}
    S(\mathbf{X}(t))= D_v^{-\frac{1}{2}}HK'(\mathbf{X}(t))H^{T}D_v^{-\frac{1}{2}},
    \label{19}
\end{equation}
\begin{equation}
K'(\mathbf{X}(t))=\operatorname{diag}(\kappa'_{1}(t),\kappa'_{2}(t)...\kappa'_{m}(t)),
\label{20}
\end{equation}
\begin{equation}
C_{ij}(\mathbf{X}(t))=cos(\mathbf{x}_{i}(t),\mathbf{x}_{j}(t)).
\label{21}
\end{equation}

Here, $H$ represent the incidence matrix of hypergraph, $D_v$ represent the vertex degree matrix. The  node feature update process  in Equation (\ref{ricci flow}) and (\ref{ricci flow matrix}) mitigates over-smoothing via a geometry-aware adaptive mechanism. It aggregates not the raw features of neighboring nodes, but rather the feature dissimilarities between neighboring and central nodes. The sign of the curvature is leveraged to dynamically modulate the direction of this aggregation: in homophilous communities with positive curvature, it reduces feature discrepancies to achieve local smoothing, whereas across heterophilous bridges with negative curvature, it amplifies these differences to preserve node distinctiveness. This bidirectional regulatory capacity, combining both positive and negative effects, transcends the unidirectional averaging process inherent to conventional HGNNs. By preserving the uniqueness of critical nodes, it fundamentally inhibits the convergence of all node features towards uniformity, thereby addressing the over-smoothing problem.

\textbf{Existence and Uniqueness of Solutions:}
The Equation (\ref{ricci flow matrix}) offers an innovative framework for modeling hypergraph feature evolution. To ensure its validity, we conduct a rigorous theoretical analysis, from which we derive the following conclusion:
\begin{Theorem}
\label{exit}
Suppose the following conditions hold:
    \begin{enumerate}
        \item For every node $i$, the degree  of node $i$  is strictly bounded: $0 < d_{\min} \leq d_i \leq d_{\max} < \infty$, where $d_{\min}$ and $d_{\max}$ denote the minimum and maximum node degrees of the hypergraph.
        
        \item  $\lambda^j(\cdot)$ is bounded on bounded sets and locally Lipschitz.
        \item For every $e$, $\kappa_e(t)$ and $w_e(t)$ are continuously bounded.
        \item  For every $e$, $1 -  (\mathbf{x}_i(t)^{T} m_{ie} )^2 \geq \varepsilon \in (0, 1]$.
    \end{enumerate}
    Then, for any initial condition $\mathbf{x}(0) \in \mathbb{R}^n$ and any $T > 0$, there exists a unique solution $\mathbf{x}(t) \in C^1([0, T], \mathbb{R}^n)$ to Equation (\ref{ricci flow matrix}).
\end{Theorem}
Theorem \ref{exit}   guarantees the reliable local evolution behavior of the dynamical system under study within a strict mathematical framework. Specifically, the local solution $\mathbf{x}(t)$ of Equation (\ref{ricci flow matrix}) is guaranteed to exist and be continuously differentiable, which ensures that the system is well-defined in the short term. Simultaneously, the theorem guarantees the solution is unique which ensures that the system trajectory emanating from any initial condition $\mathbf{x}(0)$ is one-of-a-kind, thereby lending reliability to analyses and predictions based on this equation. The full
proof is available in the supplementary material.

\subsection{Ricci Flow-guided  Hypergraph Neural Diffusion}
As shown in Equation (\ref{ricci flow}), the evolution of node features is tightly coupled with the aggregation weight function $\kappa'_{e}(t)$. However, in practice, the exact computation of this function for all hyperedges is computationally prohibitive. Therefore, to address this challenge, we propose leveraging the function approximation capabilities of a neural network to fit the hyperedge function. This leads us to propose the Ricci Flow-guided  Hypergraph Neural Diffusion(RFHND) model, a framework for node representation evolution inspired by Ricci flow. Based on the universal approximation theorem~\cite{cybenko1989approximation} of neural networks, we can draw the following conclusions.
\begin{Theorem}
\label{thm:b9}
Suppose  for any $j\in e$, $\lambda^j(\cdot)$ is a continuous function, When Forman-Ricci curvature $\kappa^{FR}$, Ollivier-Ricci curvature $\kappa^{OR}$  are used as the definition of hyperedge curvature in RFHND, there exists an HyperNet that can approximate the aggregation weight of RFHND $\kappa'_{e}(t)$ with arbitrarily high precision. 
\end{Theorem}

\begin{proof}
The goal is to demonstrate that the aggregation weight $\kappa'_{e}(t)$ is a continuous function of the node features $\mathbf{X}(t)$, which in turn allows for its approximation by a neural network.

We begin with the premise that for any node $j \in e$, its associated function $\lambda^j(\cdot)$ is continuous. As established in the supplementary material, the curvature $\kappa_e(t)$ is a continuous function of the node features. Similarly, the hyperedge weight $w_e(t)$, defined in Equation~(\ref{weight}), is also a continuous function of the node features.

The expression of $\kappa'_{e}(t)$ is a composition of the functions $\kappa_e(t)$, $w_e(t)$, and $\lambda^j(\cdot)$ through arithmetic operations that preserve continuity (assuming a non-zero denominator). Since each of its constituent functions is continuous with respect to the node features $\mathbf{X}(t)$, it follows that $\kappa'_{e}(t)$ is also a continuous function of $\mathbf{X}(t)$.

Therefore, by the Universal Approximation Theorem, there exists a neural network with sufficient capacity that can approximate the aggregation weight $\kappa'_{e}(t)$ from the node features with arbitrary precision.
\end{proof}
This theoretical result provides a rigorous foundation for our architectural design. It guarantees that parameterizing the curvature-guided aggregation weights via neural networks (e.g., MLPs) is a valid approach. In practical implementation, we drive feature updates by numerically solving differential equations.  The pseudocode for the forward propagation of RFHND is shown in Algorithm \ref{alg:hyper_ode_evolution}. To ensure the robustness of the solution process, we further derive the stability conditions for the explicit Euler method, as detailed below:
\begin{Theorem}
\label{thm:st}
The explicit Euler method for
the RFHND is stable if the step size  $\tau \leq \frac{1}{max_{i}\sum_{j}s_{ij}}.$
\end{Theorem}
Here, $s_{ij}$ denotes the elements of the matrix $S(\mathbf{X}(t))$. To ensure numerical stability and prevent divergence, the step size of the explicit method must be subject to specific constraints. We formally establish the rigorous form of this criterion in Theorem \ref{thm:st}: specifically, the step size $\tau$ must be less than the reciprocal of the maximum row sum of the hypergraph edge weight matrix to guarantee the asymptotic stability of the algorithm. For the detailed mathematical proof, please refer to  the supplementary material.

\textbf{Computational complexity.} To evaluate the computational efficiency of RFHDN, we analyze the time complexity of its forward propagation. The overall time complexity is  $  O(n d_{\text{in}} d_{\text{out}}) + T \cdot O(m (r^2 d_{\text{in}} + d^2_{\text{in}}))$, where $n$ and $m$ denote the number of nodes and hyperedges, respectively. $d_{\text{in}}$ and $d_{\text{out}}$ represent the dimensions of the input and output features, $T$ indicates the number of layers, and $r$ represents the maximum hyperedge degree. The detailed analysis is provided in the  the supplementary material.
\begin{algorithm}[htp]
\caption{Ricci Flow-guided Hypergraph Neural Diffusion}
\label{alg:hyper_ode_evolution}
\small
\textbf{Input:} Hypergraph $\mathcal{H}=(\mathcal{V}, \mathcal{E})$ with incidence matrix $H$, edge index $E$ ,degree matrix $D_v$, features matrix $\mathbf{F} \in \mathbb{R}^{n \times d_{in}}$,
   step size $\tau$, depth $T$, hidden dimension $d$.
\begin{algorithmic}[1]    
\State $\mathbf{X}{(0)} \leftarrow \operatorname{Linear}(\mathbf{F})$
    \State $(\text{node\_idx}, \text{edge\_idx}) \leftarrow E$
    
    \For{$t = 0$ \textbf{to} $T-1$}
        \State $E_{\text{edge}} \leftarrow \operatorname{scatter}(\mathbf{X}{(t)}, \text{edge\_idx})$
        \State $E_{\text{node}} \leftarrow \operatorname{scatter}(\operatorname{MLP}_{\theta_1}(E_{\text{edge}}), \text{edge\_idx})$
        \State $E_{\text{edge}} \leftarrow \operatorname{scatter}(E_{\text{edge}}, \text{edge\_idx})$
        
        \State $K^{\prime}(t) \leftarrow \operatorname{diag}(\operatorname{MLP}_{\theta_2}(E_{\text{edge}}))$
         \State Compute  Matrix $S(\mathbf{X}(t))$ by Eq.~(\ref{19})
       
    \State$F(\mathbf{X}(t)) \leftarrow \operatorname{diag}\Big( \big( S(\mathbf{X}(t)) \odot C(\mathbf{X}(t)) \big) \mathbf{1}_{N} \Big) - S(\mathbf{X}(t))$
    \State  $\mathbf{X}(t+1) \leftarrow \mathbf{X}(t) - \tau F(\mathbf{X}(t))\mathbf{X}(t)$
    \EndFor
    
    \State $Y \leftarrow \operatorname{Linear}(\mathbf{X}{(T)})$
    \State Compute loss and back propagation.
\end{algorithmic}
\end{algorithm}

\begin{table*}[htbp]
  \centering
  \caption{Dataset Statistics Summary\label{tab:dataset_stats}}
  \setlength{\tabcolsep}{4pt} 
  \begin{tabular}{l *{12}{S[table-format=5.0]}} 
    \toprule
    {\textbf{Metric}} & {\textbf{Cora}} & {\textbf{Citeseer}} & {\textbf{Pubmed}} & {\textbf{Cora-CA}} & {\textbf{DBLP-CA}} & {\textbf{Zoo}}  & {\textbf{NTU2012}} & {\textbf{ModelNet40}} &{\textbf{Walmart}}&  {\textbf{Senate}} & {\textbf{House}}  \\
    \midrule
    $|V|$ & 2708 & 3312 & 19177 & 2708 & 41302 & 101 & 2012 & 12311 &88860& 282& 1290  \\
    $|E|$ & 1579 & 1079 & 7963 & 1072 & 22363 & 43 &  2012 & 12311 &69906& 315 & 340 \\
    \#features
    & 1433 & 3703 & 500 & 1433 & 1425 & 16 & 100  & 100 &100& 100 & 100  \\
    \#classes & 7 & 6 & 3 & 7 & 6 & 7  &  67 & 40 &11& 2 & 2\\
    \bottomrule
  \end{tabular}
  \vspace{0.2cm}
\end{table*}

\section{Experiment}
\subsection{Results on Benchmark Datasets}
\textbf{Datasets}. To validate the performance of RFHND, we conducted a comprehensive evaluation on a diverse collection of benchmark datasets representing both academic and real-world scenarios.

The academic scenarios evaluation was performed on five well-established hypergraph benchmarks from co-citation and co-authorship networks: Cora, Citeseer, Pubmed, Cora-CA, and DBLP-CA~\cite{yadati2019hypergcn}. For these datasets, node features are derived from bag-of-words representations, while labels correspond to the paper's subject category. To assess the method's generalizability across different domains, we also utilized several real-world datasets. These include  Zoo from the UCI repository~\cite{dua2017uci}, the 3D vision datasets ModelNet40~\cite{wu20153d} and NTU2012~\cite{chen2003visual}, and the transaction  datasets Walmart ~\cite{amburg2020clustering}, the social network datasets House~\cite{chodrow2021generative} and Senate~\cite{fowler2006connecting}.
Following prior works, hypergraph structures were constructed for all datasets, in cases where inherent node features were absent, we initialized them using Gaussian random vectors. For the evaluation, a consistent 50\%/25\%/25\% split was used for the training, validation, and testing sets. To ensure robust and reliable results, the final reported performance is the aggregated outcome of 20 independent trials using different random splits. The statistical summary of all datasets is available in Table \ref{tab:dataset_stats}.

\begin{table*}[htbp!]
\centering
\caption{Performance comparison on academic hypergraph datasets (Mean accuracy (\%) ± standard deviation), with the best results in bold, second-best and third-best are marked with an underline.}
\label{tab2: academic datasets}
\begin{tabular}{ccccccc}
\hline
 \textbf{Models}& \textbf{Cora} & \textbf{Citeseer} & \textbf{Pubmed}&\textbf{Cora-CA}&\textbf{DBLP-CA}&\textbf{Rank}$\downarrow$ \\
\hline
HNHN & 76.36$\pm$1.92 & 72.64$\pm$1.57 & 86.90$\pm$0.30 & 77.19$\pm$1.49 & 86.78$\pm$0.29& 12\\
HGNN & 79.39$\pm$1.36 & 72.45$\pm$1.16 & 86.44$\pm$0.44 & 82.64$\pm$1.65 & 91.03$\pm$0.20 &8 \\
 HCHA& 79.14$\pm$1.02 & 72.42$\pm$1.42 & 86.41$\pm$0.36 & 82.55$\pm$0.97 & 90.92$\pm$0.22 & 11\\
HyperGCN & 78.45$\pm$1.26 & 71.28$\pm$0.82 & 82.84$\pm$8.67 & 79.48$\pm$2.08 & 89.38$\pm$0.25&13 \\
UniGCNII& 78.81$\pm$1.05 & 73.05$\pm$2.21 & 88.25$\pm$0.40 & 83.60$\pm$1.14 & \underline{91.69$\pm$0.19} &6\\
AllSetTransformer& 78.59$\pm$1.47 & 73.08$\pm$1.20 & 88.72$\pm$0.37 & 83.63$\pm$1.47 & {\underline{91.53$\pm$0.23} } &5\\
AllDeepSets & 76.88$\pm$1.80 & 70.83$\pm$1.63 & {88.75$\pm$0.33 }& 81.97$\pm$1.50 & 91.27$\pm$0.27 &10 \\
HAN & 79.70$\pm$1.77 & 74.12$\pm$1.52 & 85.32$\pm$2.25 & 81.71$\pm$1.73 & 90.17$\pm$0.65  &9\\
ED-HNN & 80.31$\pm$1.35 & 73.70$\pm$1.38 & 
{\textbf{89.03$\pm$0.53}} & 83.97$\pm$1.55 & 
{\textbf{91.90$\pm$0.19}} &3\\
HyperGINE & 79.26$\pm$0.41 & 73.72$\pm$0.52 & 87.91$\pm$0.28 & 82.88$\pm$0.48 & - &11\\
KHGNN& \underline{80.67$\pm$0.76} & {\underline{74.80$\pm$1.10} }& 88.47$\pm$0.47 & {\underline{84.25$\pm$0.74}} &-&4\\
FrameHGNN & {\underline{81.51$\pm$0.99}} & {\underline{74.72$\pm$2.10}} &{\underline{88.73$\pm$0.42}} & {\underline{85.18$\pm$0.69}} &-&{2}\\
\hline
RFHND & 
{\textbf{81.83$\pm$1.48}} & {\textbf{76.04$\pm$1.43}} & {  \underline{88.88$\pm$0.35}}&
{\textbf{85.40$\pm$1.05}} & 91.20$\pm$0.41&{\textbf{1}}\\
\hline
\end{tabular}
\end{table*}

\begin{table*}[htbp!]
\centering
\caption{Performance comparison on real-world hypergraph datasets (Mean accuracy (\%) ± standard deviation), with the best results in bold, second-best and third-best are marked with an underline.}
\label{tab3: real-world datasets}
\begin{tabular}{cccccccc}
\hline
 \textbf{Models}&  \textbf{Zoo}&\textbf{ NTU2012 }& \textbf{ModelNet40}&\textbf{Walmart}
& \textbf{Senate}&\textbf{House}&\textbf{Rank}$\downarrow$ \\
\hline
HNHN & 93.59$\pm$5.88 & 89.11$\pm$1.44&97.84$\pm$1.25&{\underline{81.35$\pm$0.61}}& 50.93$\pm$6.33&67.80$\pm$2.59 &6\\
HGNN & 92.50$\pm$4.58 & 87.72$\pm$1.35 & 95.44$\pm$0.33 &80.33$\pm$0.42& 48.59$\pm$4.52&61.39$\pm$2.96&11\\
HCHA & 93.65$\pm$6.15 & 87.48$\pm$1.87& 94.48$\pm$0.28&80.33$\pm$0.80 & 48.62$\pm$4.41&61.36$\pm$2.53&10 \\
HyperGCN & N/A & 56.36$\pm$4.86 & 75.89$\pm$5.26&81.05$\pm$0.59 & 42.45$\pm$3.67&48.32$\pm$2.93&13\\
UniGCNII & 93.65$\pm$4.37 & 89.30$\pm$1.33 & 98.07$\pm$0.23&81.12$\pm$0.67 & 49.30$\pm$4.25&61.70$\pm$3.37&7\\
AllSetTransformer & {\underline{97.50$\pm$3.59}}& 88.69$\pm$1.24&98.20$\pm$0.20&{\underline{81.38$\pm$0.58}}& 51.83$\pm$5.22 &69.33$\pm$2.20&4 \\
AllDeepSets & {\underline{95.39$\pm$4.77}} & 88.09$\pm$1.52&96.98$\pm$0.26&81.06$\pm$0.54& 48.17$\pm$5.67&67.82$\pm$2.40&8 \\
HAN & 75.77$\pm$7.10 & 83.58$\pm$1.46 & 94.04$\pm$0.41& 79.72$\pm$0.62 & - & 62.00$\pm$9.09&12 \\
ED-HNN & -  & 88.67$\pm$0.92 & 97.83$\pm$0.33 & -& {\underline{64.79$\pm$5.14}}&{\underline{72.45$\pm$2.28}}&5 \\
HyperGINE & -  & 88.52$\pm$0.42 & 97.61$\pm$0.16 & -& - & - &9\\
KHGNN & -  & {\underline{89.60$\pm$1.64}} & {\underline{98.33$\pm$0.14}} & -& - & -&{3}\\
FrameHGNN  &- &{\underline{89.98$\pm$2.02}}& {\underline{98.41$\pm$0.18}} & - &{\underline{67.61$\pm$5.27} }& {\underline{72.82$\pm$2.22}} &{{2}}\\
\hline
RFHND & 
{\textbf{97.70$\pm$2.14}} & 
{\textbf{93.36$\pm$1.18}} & {\textbf{98.56$\pm$0.16}}& 
 {\textbf{82.40$\pm$1.40}} & {\textbf{68.12$\pm$4.98}} & {\textbf{73.52$\pm$2.25}}&{\textbf{1}} \\
\hline
\end{tabular}
\end{table*}

\textbf{Baselines.}
To evaluate the  performance of RFHND, we compared it against a comprehensive set of representative hypergraph neural network baselines. This  includes several  methods, such as HGNN~\cite{feng2019hypergraph}, which applies spectral convolution operations designed for hypergraph data. We also include HCHA~\cite{bai2021hypergraph}, a method that incorporates hierarchical attention mechanisms for representation learning, and the widely-used HyperGCN~\cite{yadati2019hypergcn}, which extends graph convolutions to the hypergraph domain through a clique expansion strategy.

Our comparison further incorporates HNHN~\cite{dong2020hnhn}, a model noted for its novel hypergraph-specific normalization techniques, alongside UniGCNII~\cite{huang2021unignn}, which unifies multiple hypergraph convolution paradigms with residual connections. The well-known HAN~\cite{wang2019heterogeneous} model is also included as a key representative of hierarchical attention networks. Additionally, we benchmark against AllSetTransformer and AllDeepSets~\cite{chien2021you}, which adapt deep set theory to ensure permutation invariance in hypergraph learning.

The evaluation is extended to more recent and specialized architectures. The baseline set features ED-HNN~\cite{wang2022equivariant}, a method that utilizes equivariant hypergraph diffusion operators. To account for long-range dependencies, we include HyperGINE and KHGNN~\cite{xie2025k}, the latter of which facilitates interactions between distant nodes via $K$-hop message passing. Finally, the multi-scale approach FrameHGNN~\cite{li2025deep} is included, which integrates framelet transforms into its architecture. To ensure a fair and standardized comparative environment, all models were implemented within the PyTorch Geometric library~\cite{fey2019fast}.

\textbf{Experiment Setting.} For the RFHND model, the Adam optimizer is employed. The main hyperparameters—learning rate, weight decay, dropout rate, hidden dimension, and total training epochs—are individually tuned for each dataset based on validation performance. Specifically, the learning rate is selected from $\left\{0.001, 0.01\right\}$, weight decay from the range $[0.001, 0.03]$, input dropout from $\left\{0.001, 0.01, 0.1, 0.2, 0.3\right\}$, hidden dimension from $\left\{16, 32, 64, 128, 256, 512\right\}$, and total training epochs from $\left\{2, 3, 4\right\}$. The optimal configuration for each dataset is determined within these ranges. When applicable, a cosine learning rate scheduler (CosineLR) is adopted. We use the methods implemented by torchdiffeq~\cite{chen2018neural} as the differential equation solver for RFHND. All experiments are conducted using a fixed random seed to ensure reproducibility. Source code is available at \href{https://gitee.com/zmy-ovo/rfhnd.git}{https://gitee.com/zmy-ovo/rfhnd.git}.

\textbf{Results.} 
As demonstrated in Table \ref{tab2: academic datasets} and Table \ref{tab3: real-world datasets}, the experimental results fully showcase the superior performance of our proposed RFHND model. By calculating the average ranking across all comparison methods, the RFHND model maintained the first position on both academic and real-world datasets, which strongly attests to its powerful generalization ability. Specifically, the RFHND model achieved State-of-the-Art performance in three out of the five academic benchmark datasets and secured an excellent second-place ranking on one of the others. This performance conclusively confirms the RFHND model's crucial capability to effectively capture complex relationships within structured data. Furthermore, the RFHND model comprehensively surpassed all comparison methods on the six real-world datasets, which span diverse application domains such as transaction networks and social network analysis. These consistent results across different dataset types collectively validate the outstanding effectiveness of the RFHND model for the hypergraph node classification task.
\begin{table*}[htbp]
\centering
\caption{Model Performance On Synthetic Hypergraphs with Controlled Heterophily $\alpha$.}
\label{tab4: SY datasets}
\begin{tabular}{cccccccc}
\toprule
\textbf{Model}& \bm{$\alpha = 1$} & \bm{$\alpha = 2$} & \bm{$\alpha = 3$} &\bm{$\alpha = 4$}& \bm{$\alpha = 5$} &\bm {$\alpha = 6$} &\bm {$\alpha = 7$} \\
\midrule
HGNN & 92.58 $\pm$ 0.57 & 87.40 $\pm$ 0.89 & 82.32 $\pm$ 0.84& 77.18 $\pm$ 1.08 & 70.06 $\pm$ 1.35 & 69.98 $\pm$ 0.85 & 68.41 $\pm$ 1.39 \\
HyperGCN  & 83.40 $\pm$ 1.67 & 78.13 $\pm$ 1.57 & 77.70 $\pm$ 1.01& 74.90 $\pm$ 1.31 & 50.15 $\pm$ 1.94 & 52.11 $\pm$ 0.75 & 49.40 $\pm$ 1.63 \\
UniGCN    & 90.42 $\pm$ 1.04 & 82.12 $\pm$ 0.67 & 80.04 $\pm$ 0.98& 76.86 $\pm$ 1.10 & 71.50 $\pm$ 1.50 & 52.30 $\pm$ 1.18 & 49.20 $\pm$ 0.74 \\
EDGNN     & 93.88 $\pm$ 0.67 & 89.89 $\pm$ 0.59 & 83.52 $\pm$ 0.85& 76.65 $\pm$ 1.21 & 72.18 $\pm$ 1.36 & 52.98 $\pm$ 0.78 & 49.65 $\pm$ 1.16 \\
RFHND   & \textbf{95.63 $\pm$ 0.62 } &\textbf{90.17 $\pm$ 0.67} & \textbf{84.21 $\pm$ 0.81}&\textbf{77.48 $\pm$ 0.98} &\textbf{74.26 $\pm$ 1.57} &\textbf{73.51 $\pm$ 1.07 }&\textbf{73.46 $\pm$ 1.35 }\\
\bottomrule
\end{tabular}
\end{table*}
\begin{figure*}[htb]
\centering  
\subfigure[Cora]
{
\includegraphics[width=5cm,height =4cm]{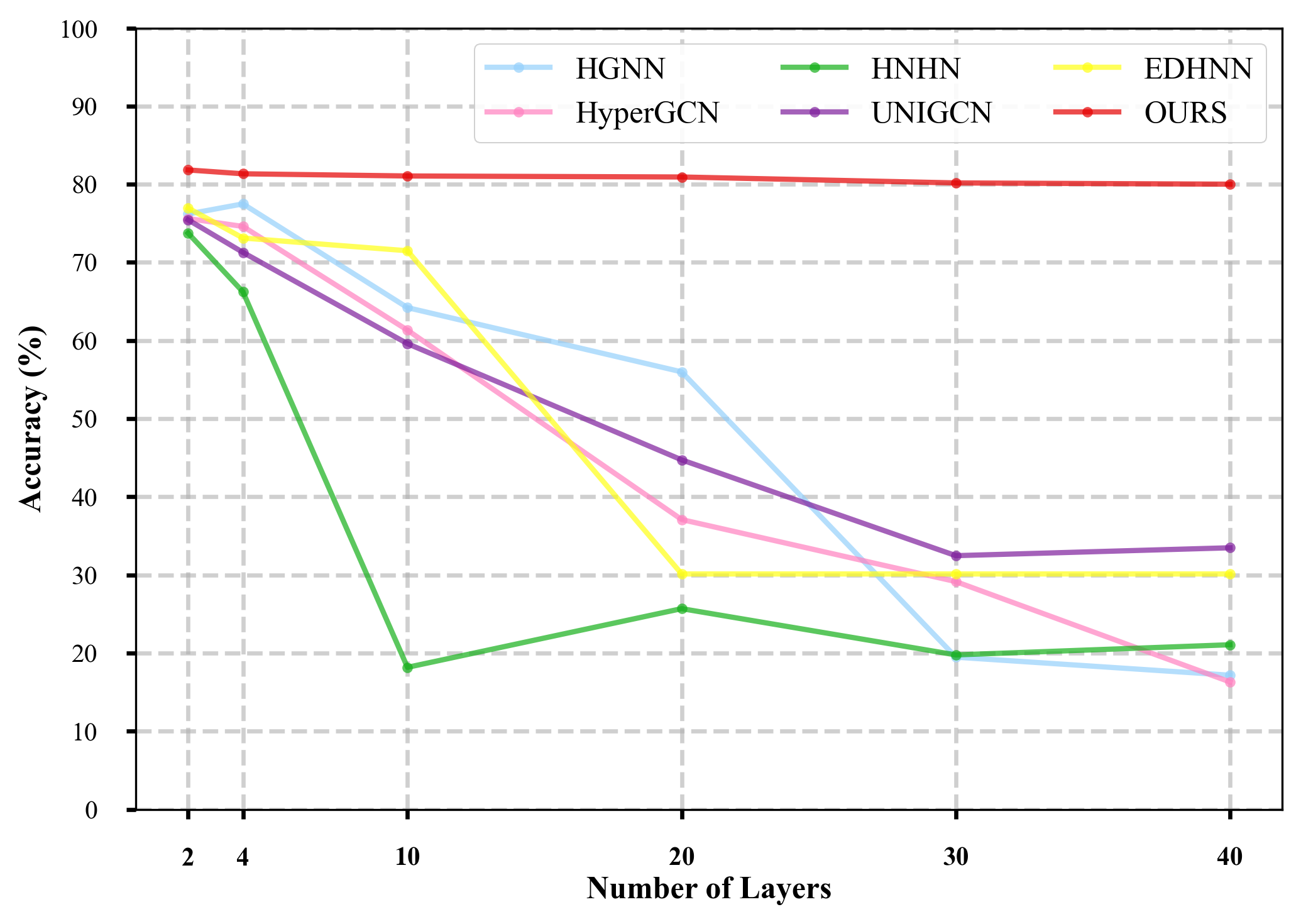}}
\subfigure[Cora-CA]
{
\includegraphics[width=5cm,height =4cm]{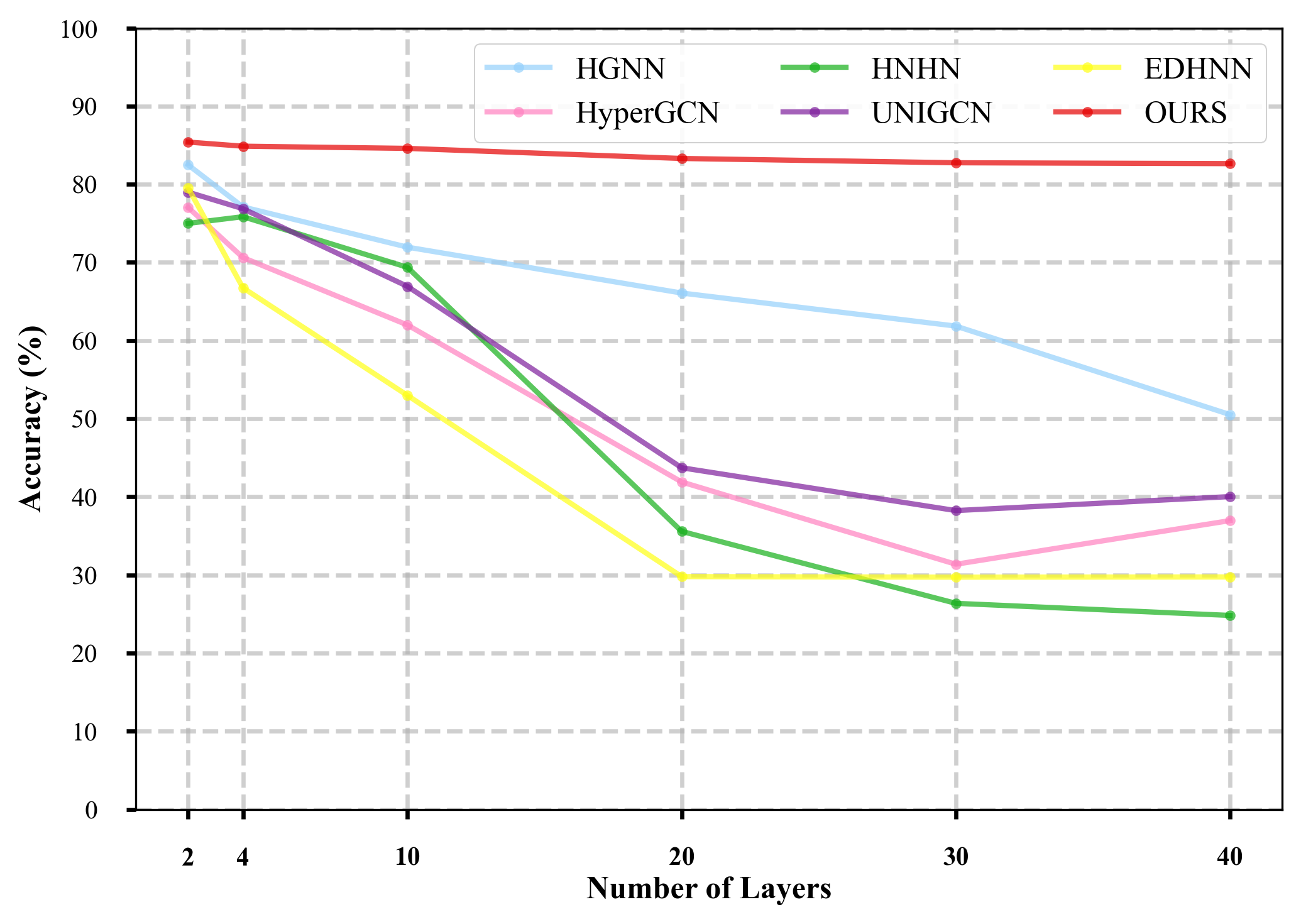}}
\subfigure[Citeseer]
{
\includegraphics[width=5cm,height =4cm]{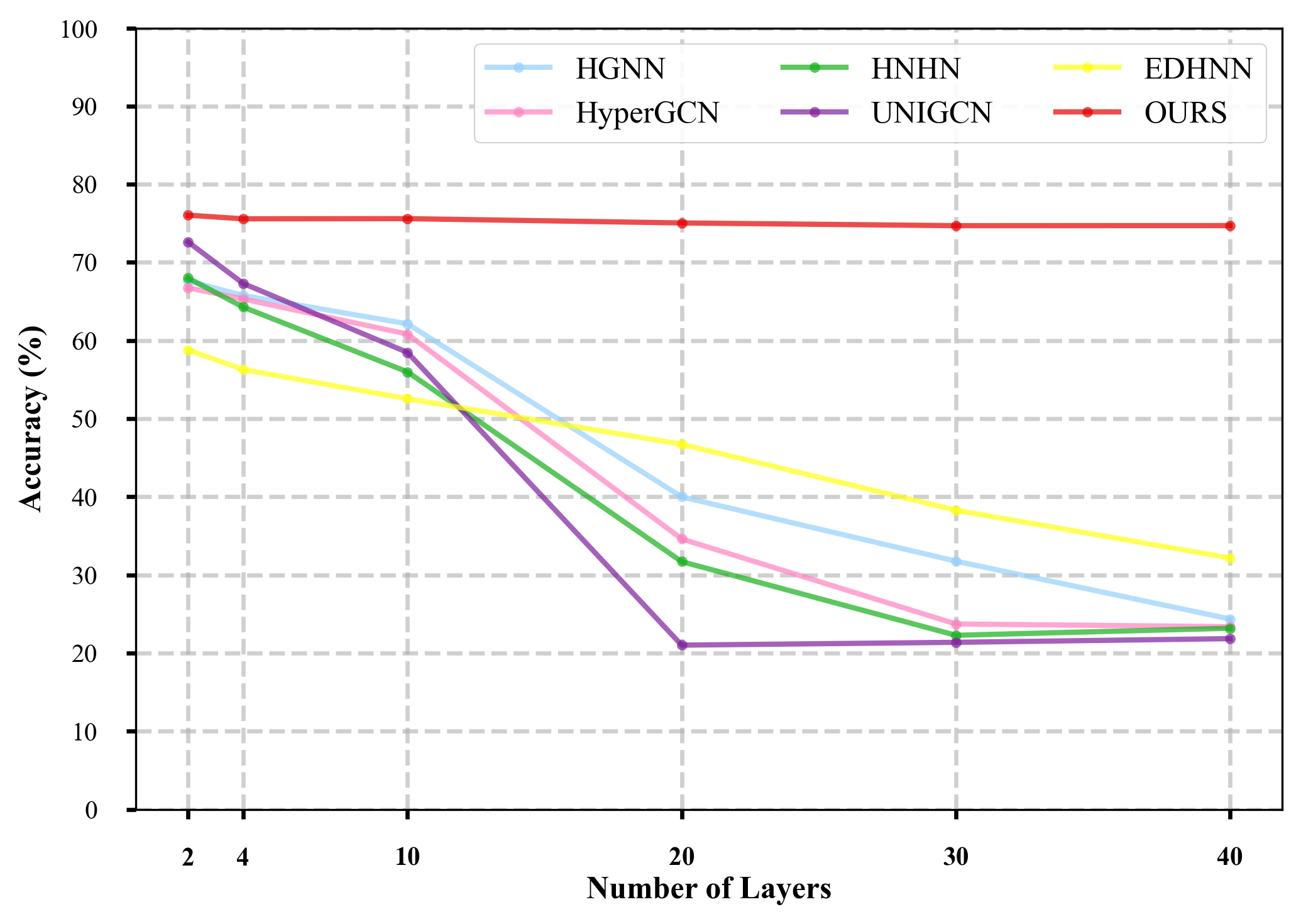}}
\subfigure[Cora]
{
\includegraphics[width=5cm,height =4cm]{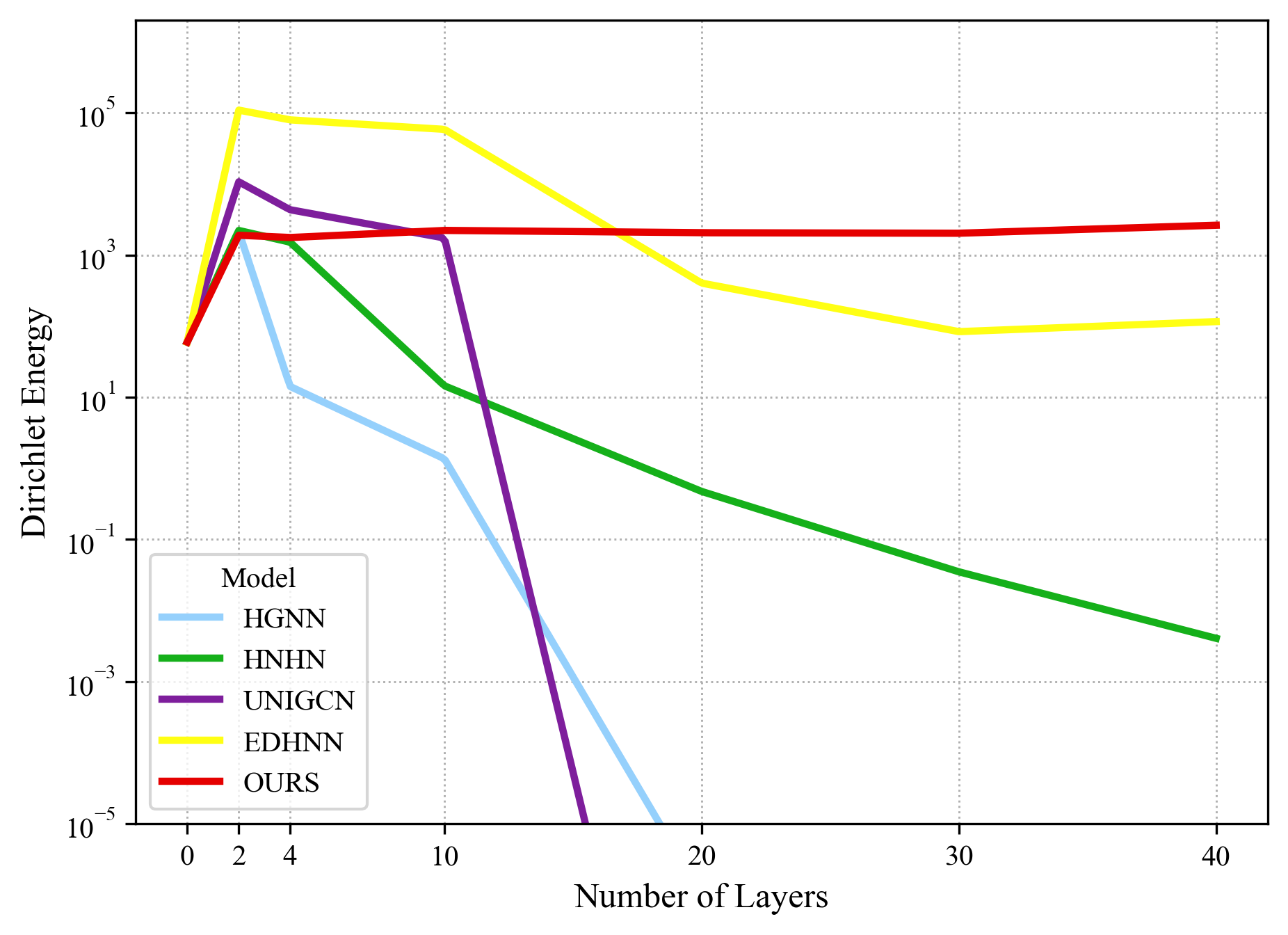}}
\subfigure[Cora-CA]
{
\includegraphics[width=5cm,height =4cm]{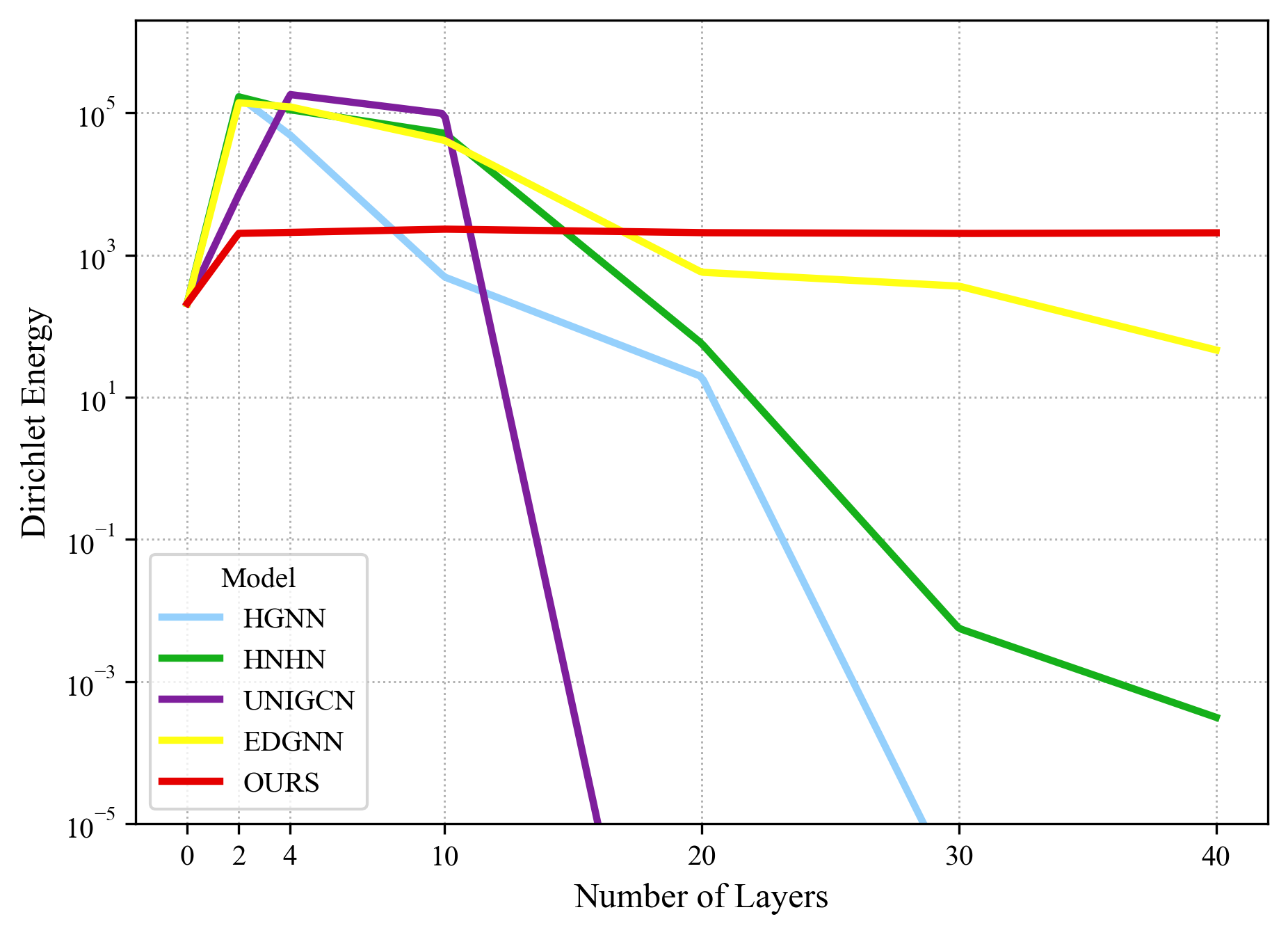}}
\subfigure[Citeseer]
{
\includegraphics[width=5cm,height =4cm]{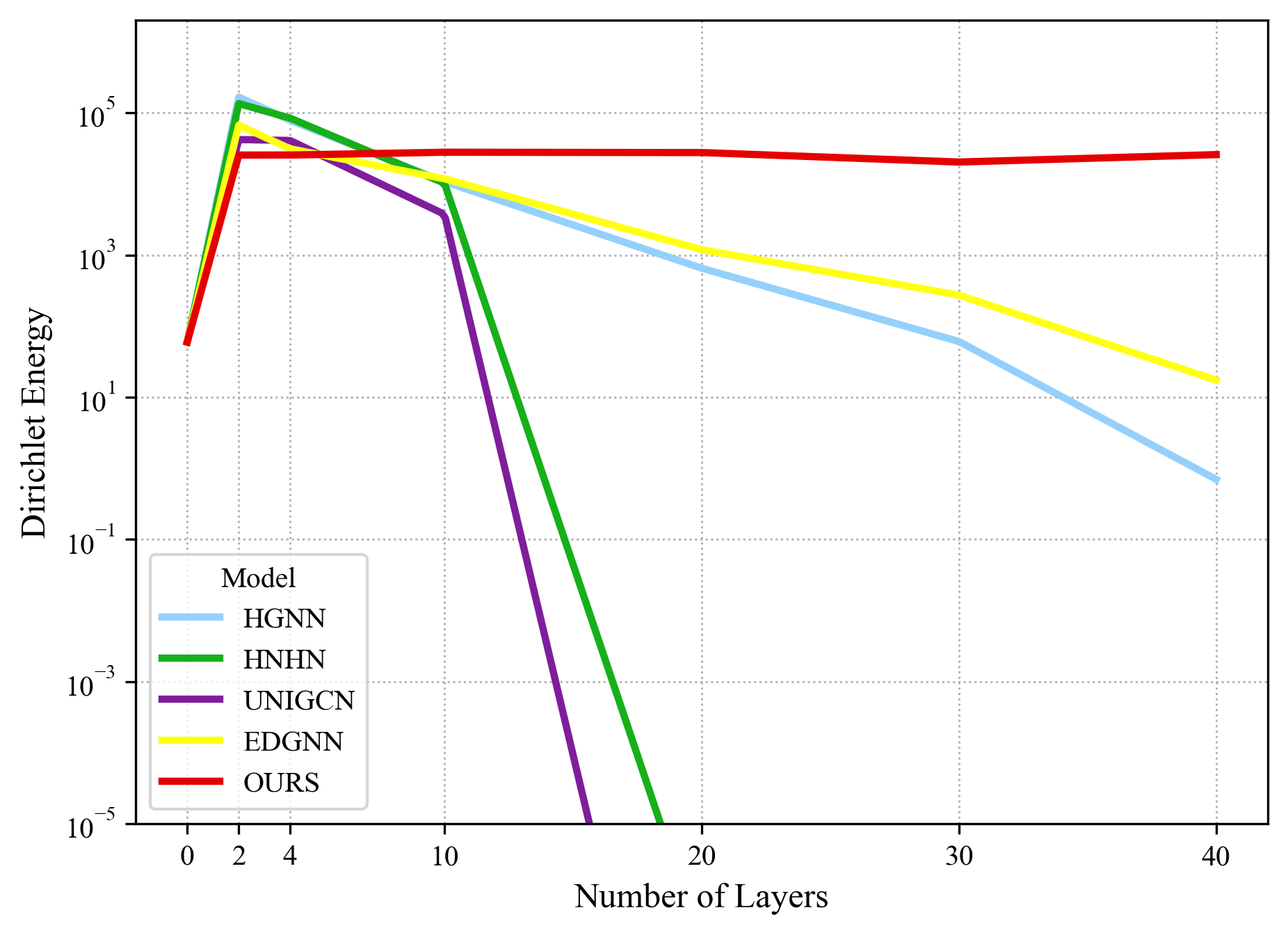}}
\caption{(a) shows the accuracy(\%)  with various
 depths on Cora. (b) shows the accuracy(\%)  with various
 depths on Cora-CA. (c) shows the accuracy(\%)  with various
 depths on Citeseer. (d) shows the dirichlet energy with various depths on Cora. (e) shows the dirichlet energy with various depths on Cora-CA. (f) shows the dirichlet energy with various depths on Citeseer.}
\label{fig:oversmoothing}
\end{figure*}

\subsection{Results on Synthetic Heterophilic Hypergraph Dataset}
\textbf{Experiment Setting.}
To assess the capability of RFHND in handling datasets exhibiting varying degrees of structural heterogeneity, we conduct experiments on synthetic datasets with controllable heterophily levels. These datasets are generated following the contextual hypergraph stochastic block model~\cite{deshpande2018contextual, ghoshdastidar2014consistency, chien2018community}. In particular, we create two node classes, each containing 2,500 nodes, and randomly sample 1,000 hyperedges. Each hyperedge is composed of 15 nodes, among which $\alpha_i$ nodes are drawn from class $i$. The heterophily level is quantified as $\alpha = \min\{\alpha_1, \alpha_2\}$.

The node features are sampled from Gaussian distributions conditioned on node labels, with a standard deviation of 1.0. We examine both homophilic configurations ($\alpha = 1, 2,3$ ) and heterophilic configurations ($\alpha = 4 \sim 7$ ). The proposed RFHND framework is benchmarked against several baselines, including HGNN, HyperGCN, and ED-HNN. Consistent with prior studies, we employ a 50\%/25\%/25\% partitioning of the data for training, validation, and testing. To ensure statistical reliability, all experiments are repeated over 10 random splits, and the final results are reported as averaged outcomes.

\textbf{Results.}
The results presented in Table \ref{tab4: SY datasets} indicate that the RFHND model consistently achieves superior performance compared to all baseline methods across the evaluated datasets. Its advantage is particularly evident in heterophilic settings (when $\alpha$ \textgreater 3), where the model demonstrates markedly improved robustness and generalization ability. These findings substantiate the effectiveness of the proposed architecture.

\subsection{Ablation Study}
In this section, we perform ablation experiments  on multiple datasets to assess the contribution of each submodule within our model. The detailed findings are summarized as follows.
\begin{table}[h!]
\centering 
\caption{Ablation Study.}
\label{tab:ab}
\begin{tabular}{lccc} 
\toprule 
\textbf{Variants} & \textbf{Zoo} & \textbf{NTU2012} & \textbf{ModelNet40} \\
\midrule 
w/o COS & 97.16 $\pm$ 2.90 & 90.04 $\pm$ 1.09 & 96.78 $\pm$ 0.27 \\
w/o HyperNet & 96.89 $\pm$ 2.87 & 91.52 $\pm$ 1.12 & 97.42 $\pm$ 0.23 \\
w/o C and H & 96.43 $\pm$ 2.90 & 89.79 $\pm$ 1.30 & 96.86 $\pm$ 0.27 \\
\midrule 
\textbf{RFHND (Full Model)} & \textbf{97.70 $\pm$ 2.14} & \textbf{93.36 $\pm$ 1.18} & \textbf{98.56 $\pm$ 0.16} \\ 
\bottomrule 
\end{tabular}
\end{table}
 
\begin{itemize} \item[$\bullet$]
\textbf{w/o COS:} In this variant of RFHND, the cosine coefficient is removed from the node feature evolution process, so feature updates are performed without cosine-based modulation.

\item[$\bullet$]
\textbf{w/o HyperNet:} In this variant of RFHND, the HyperNet used to model edge curvature is replaced with a randomly generated numerical value, instead of learning curvature adaptively.
\item[$\bullet$]
\textbf{w/o C and H:} This RFHND variant excludes both the cosine coefficient in the node feature evolution process and the HyperNet used for edge curvature modeling.
\end{itemize}

As shown in Table \ref{tab:ab}, all variants  exhibit degraded performance in label prediction, underscoring the importance of these in enhancing the modeling of interaction types. Specifically, the joint removal of the $\text{COS}$ and $\text{HyperNet}$ modules (w/o $\text{C}$ and $\text{H}$) leads to a substantially more pronounced performance drop compared to the singular ablations, underscoring the critical importance of coordination between these two modules. 

\begin{figure*}[htb]
\centering  
\subfigure[Gaussian Noise]
{
\includegraphics[width=7cm,height =5cm]{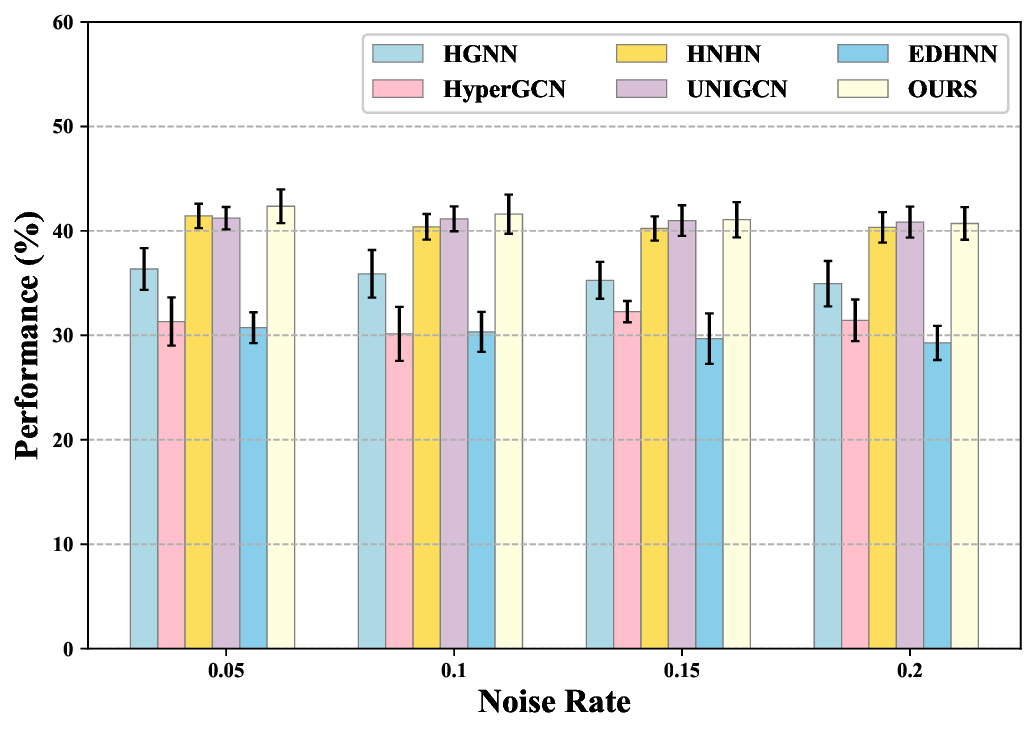}}
\subfigure[Uniform Noise]
{
\includegraphics[width=7cm,height =5cm]{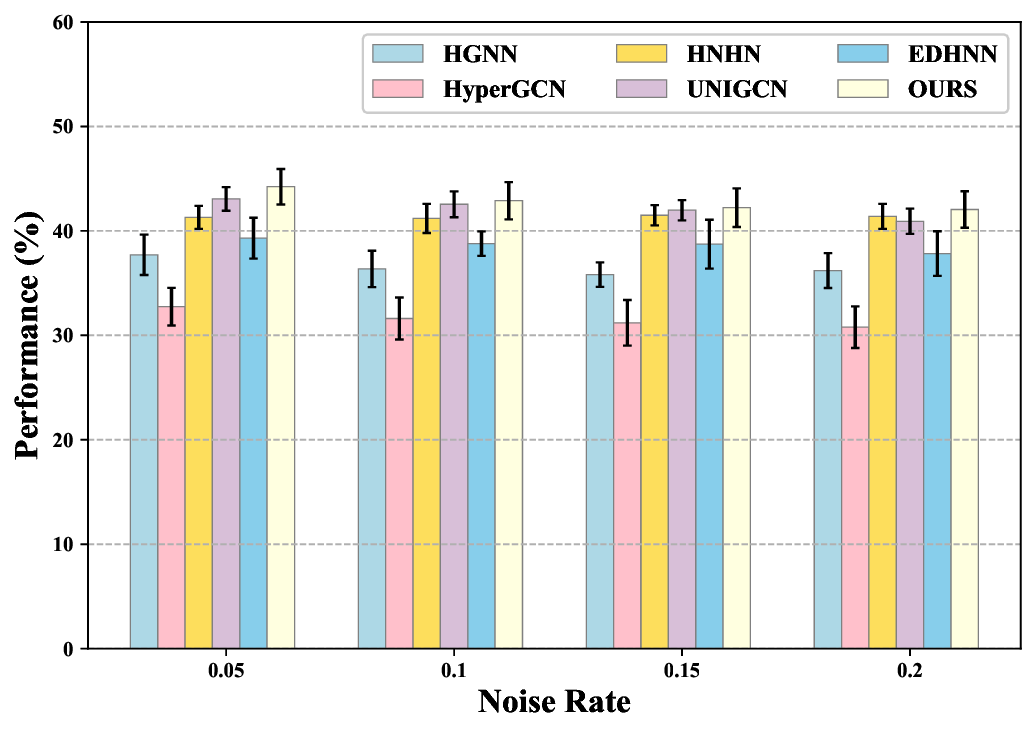}}
\subfigure[Mask Noise]
{
\includegraphics[width=7cm,height =5cm]{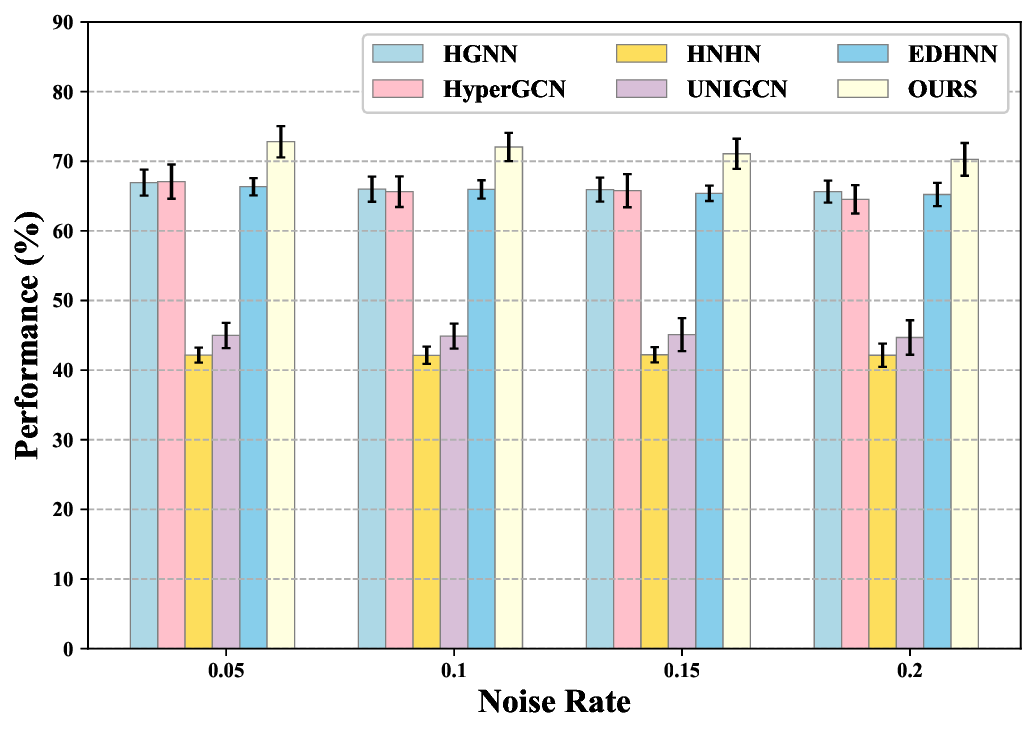}}
\subfigure[Structure Noise]
{
\includegraphics[width=7cm,height =5cm]{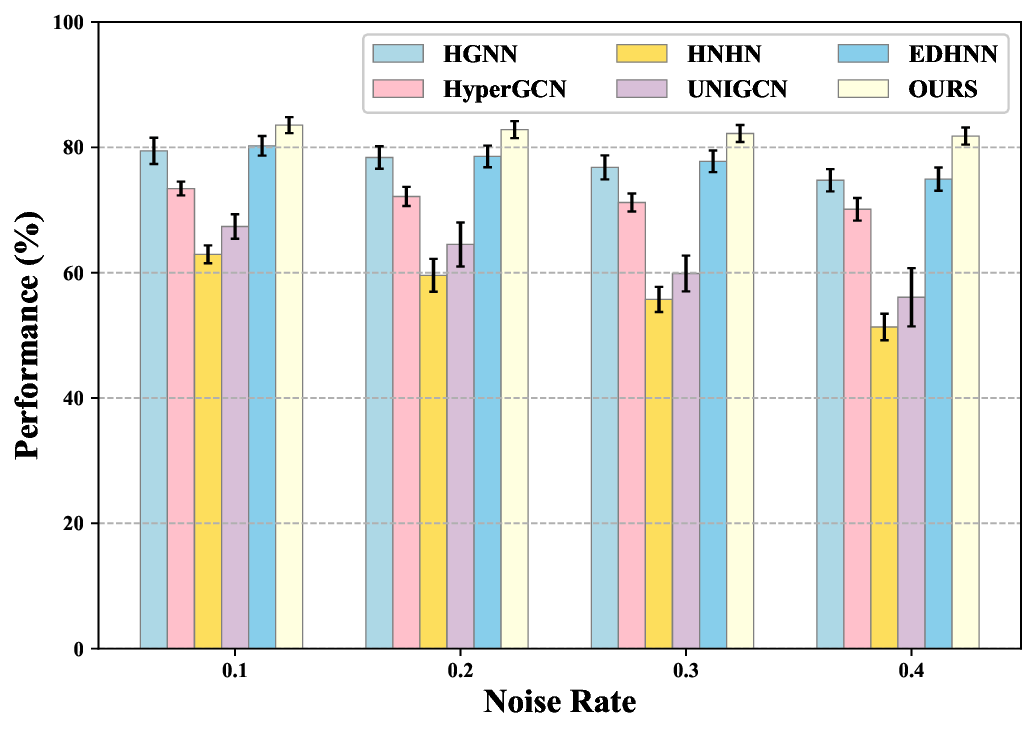}}
\caption{(a) shows the performance on Citeseer under feature  gaussian noise. (b) shows the performance on Citeseer under feature  uniform noise. (c) shows the performance on Citeseer under feature mask noise. (d) shows the performance on Cora-CA under structure noise.}
\label{noise}
\end{figure*}

\subsection{Over-smoothing analysis}
This section examines how model depth affects the performance of hypergraph neural networks. Most existing HGNNs are inherently shallow, which limits their ability to capture information from higher-order neighbors~\cite{chen2022preventing}. However, simply increasing the depth often leads to oversmoothing, where node representations become increasingly similar and eventually indistinguishable.

To study this issue, we evaluated several models with different depths on Cora, Cora-CA and Citeseer datasets. Using a fixed data split, we tested layer settings of 2, 4, 10, 20, 30, and 40 to observe how performance changes as the network becomes deeper. Meanwhile, we also recorded the Dirichlet Energy  across different layer configurations. Specifically, for zero layers, the energy is calculated using the initial node features, while for deeper configurations, it is calculated from the node embeddings output by the trained network.

As systematically illustrated in Fig.~\ref{fig:oversmoothing}, the results yield two key observations:
\begin{itemize}
 \item The proposed RFHND model consistently delivers superior performance across all tested layer depths. Crucially, its performance exhibits remarkable stability  as the number of layers is progressively increased.
 \item RFHND shows a stable Dirichlet Energy trajectory on above three datasets, while comparison methods experience a rapid energy decline as depth increases. This is consistent with our theoretical analysis in Section \ref{th:de}.
\end{itemize}

\begin{figure*}[htb]
\centering  
\subfigure[]
{
\includegraphics[width=7cm,height =5cm]{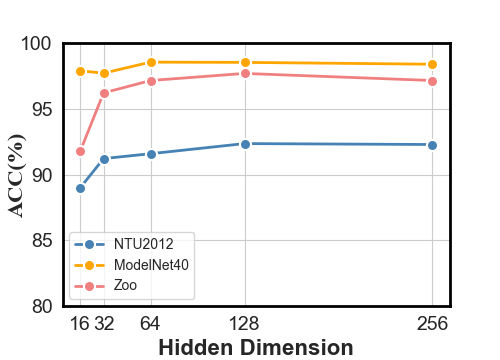}}
\subfigure[]
{
\includegraphics[width=7cm,height =5cm]{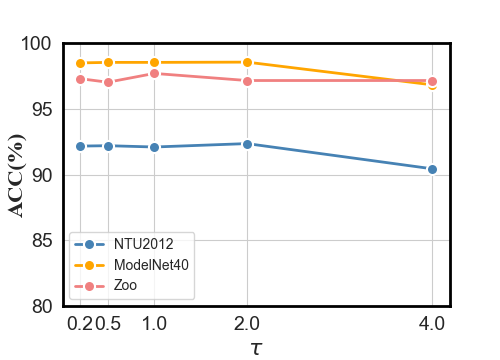}}
\caption{(a) shows RFHND accuracy on  NTU2012 and ModelNet40 with various hidden dimensions.  (b) shows RFHND accuracy on NTU2012 and ModelNet40 with various $\tau$. }
\label{para}
\end{figure*}
\begin{figure*}[htp]
\centering  
\subfigure[Pubmed(t=0)]
{
\includegraphics[width=4.5cm,height =3.6cm]
{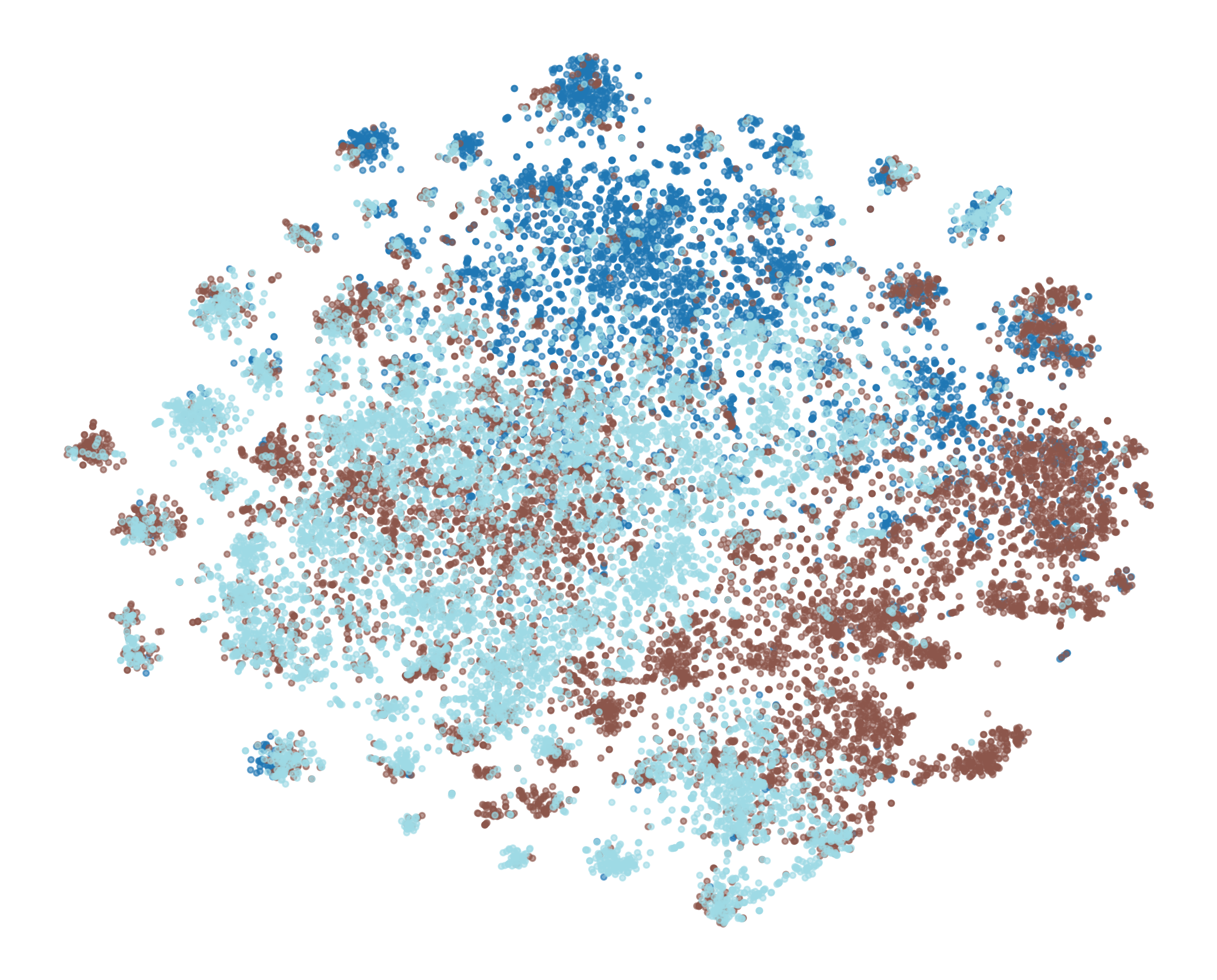}}
\subfigure[Pubmed(t=1)]
{
\includegraphics[width=4.5cm,height =3.6cm]
{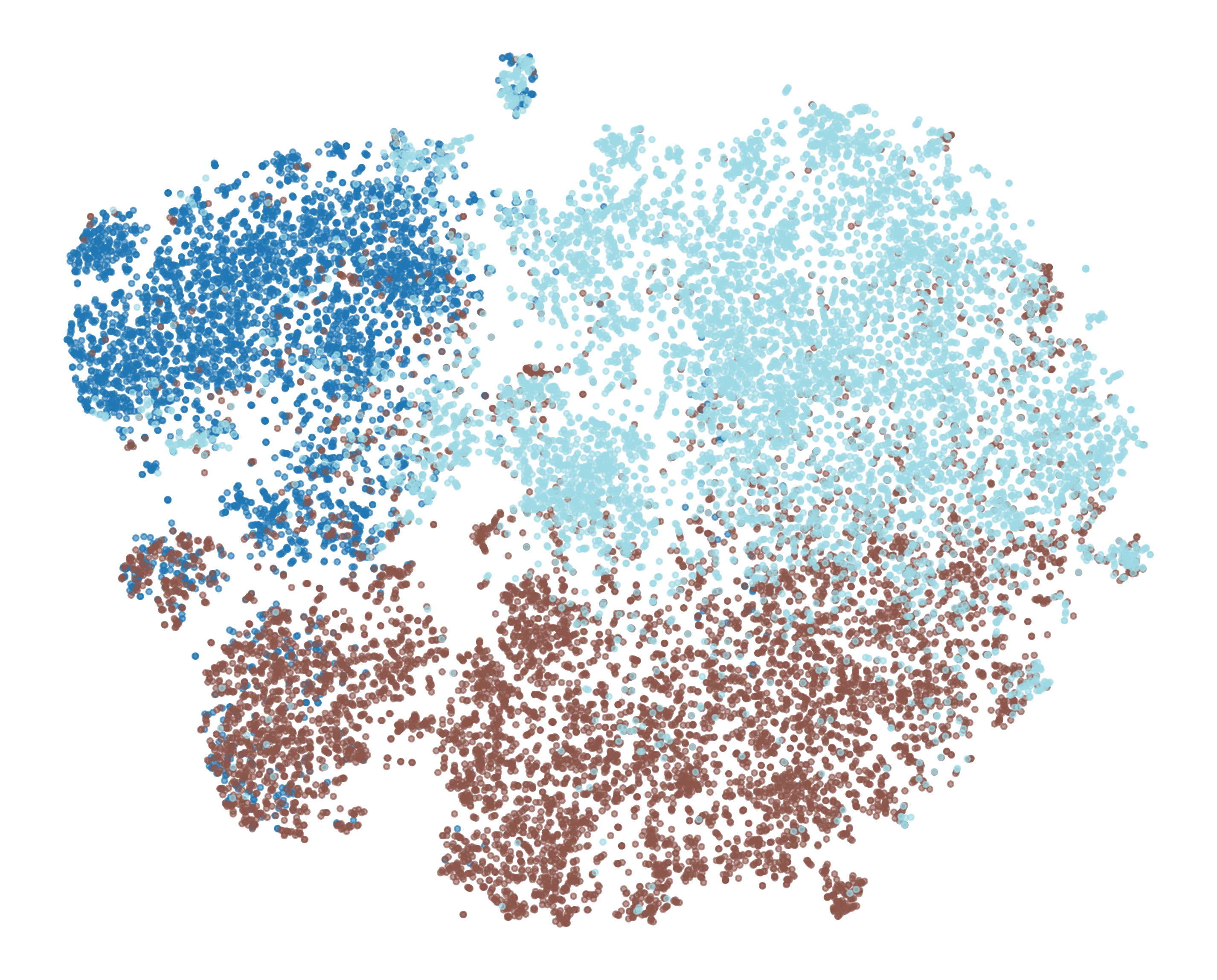}}
\subfigure[Pubmed(t=2)]
{
\includegraphics[width=4.5cm,height =3.6cm]
{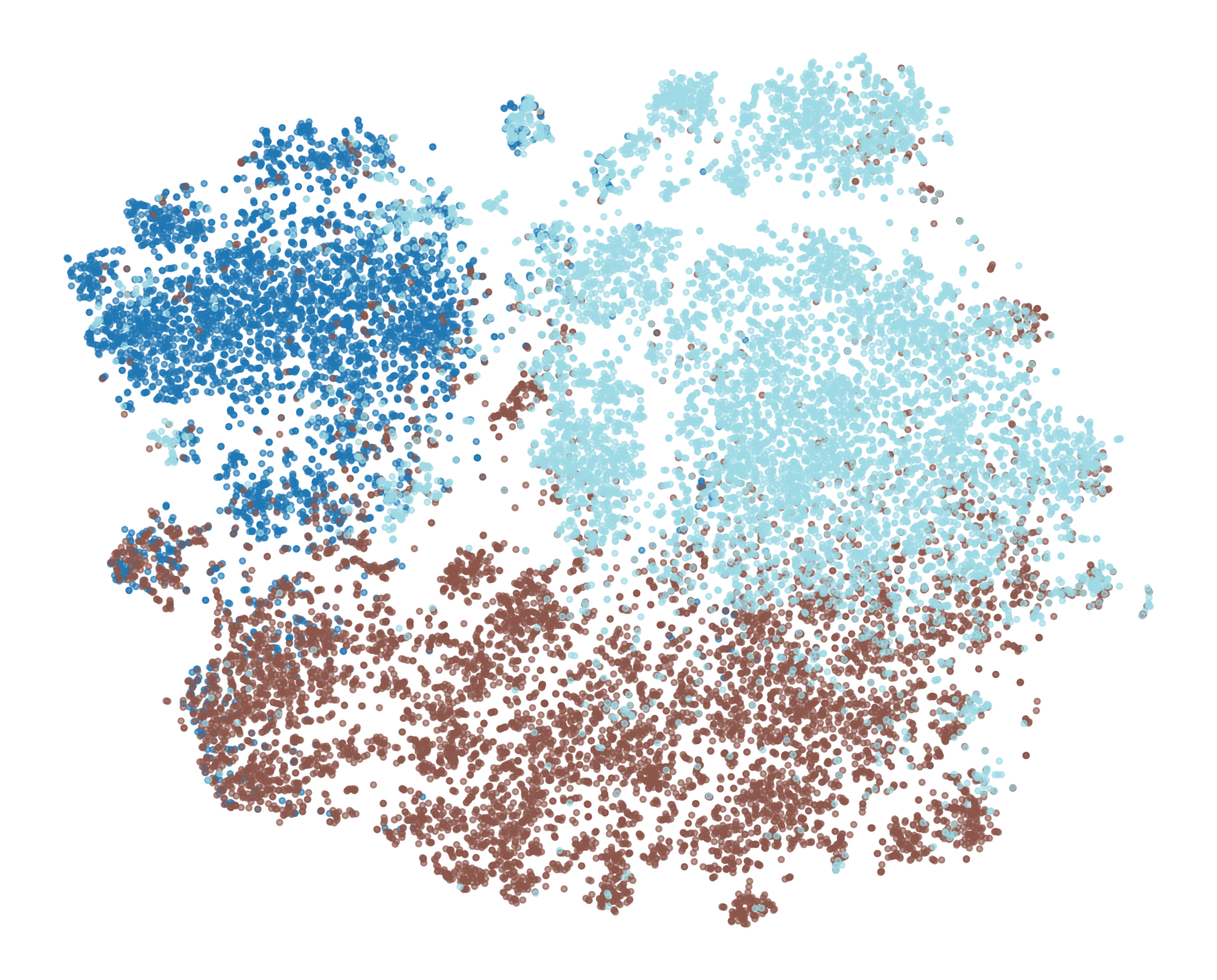}}
\subfigure[DBLP(t=0)]
{
\includegraphics[width=4.5cm,height =3.6cm]
{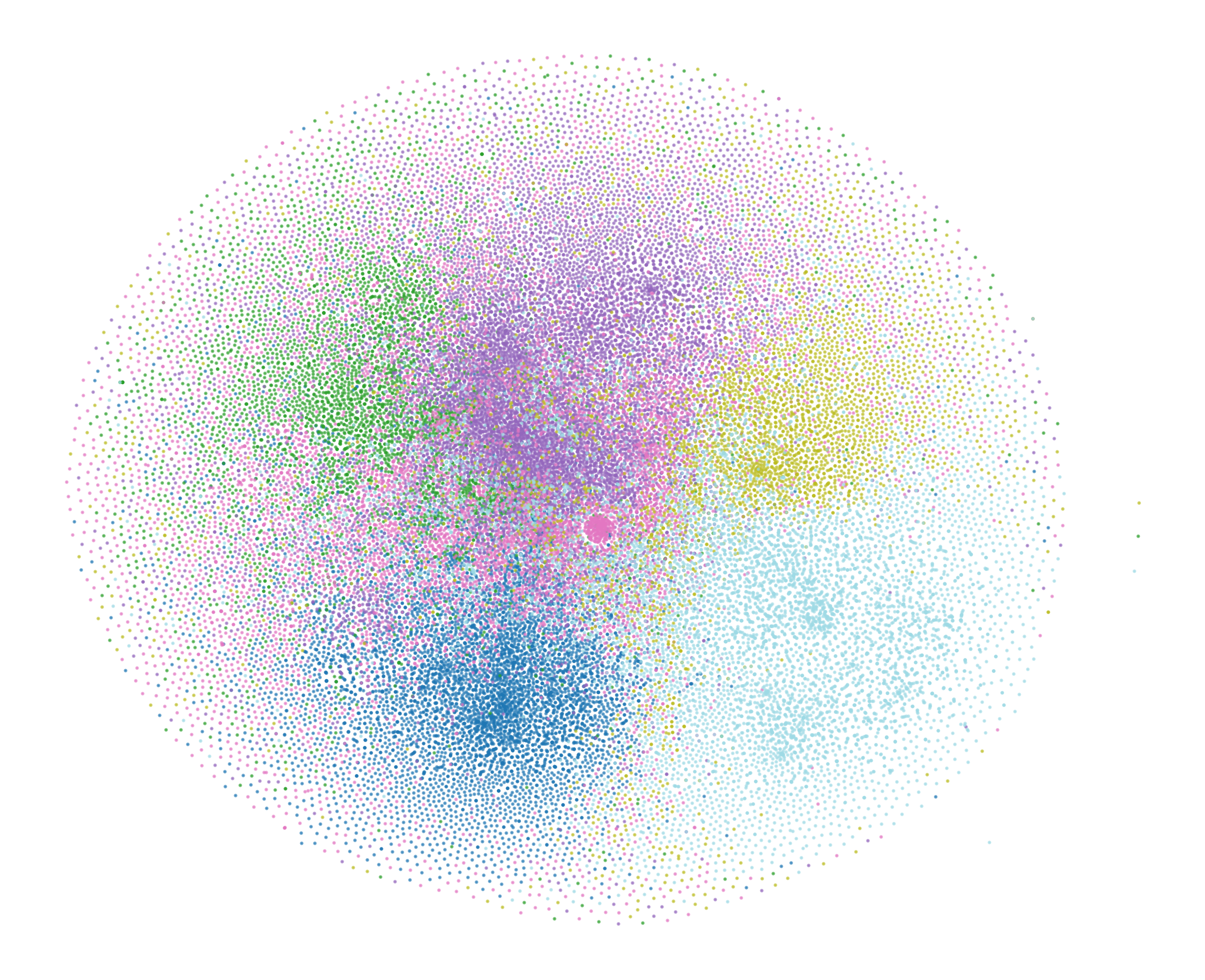}}
\subfigure[DBLP(t=1)]
{
\includegraphics[width=4.5cm,height =3.6cm]
{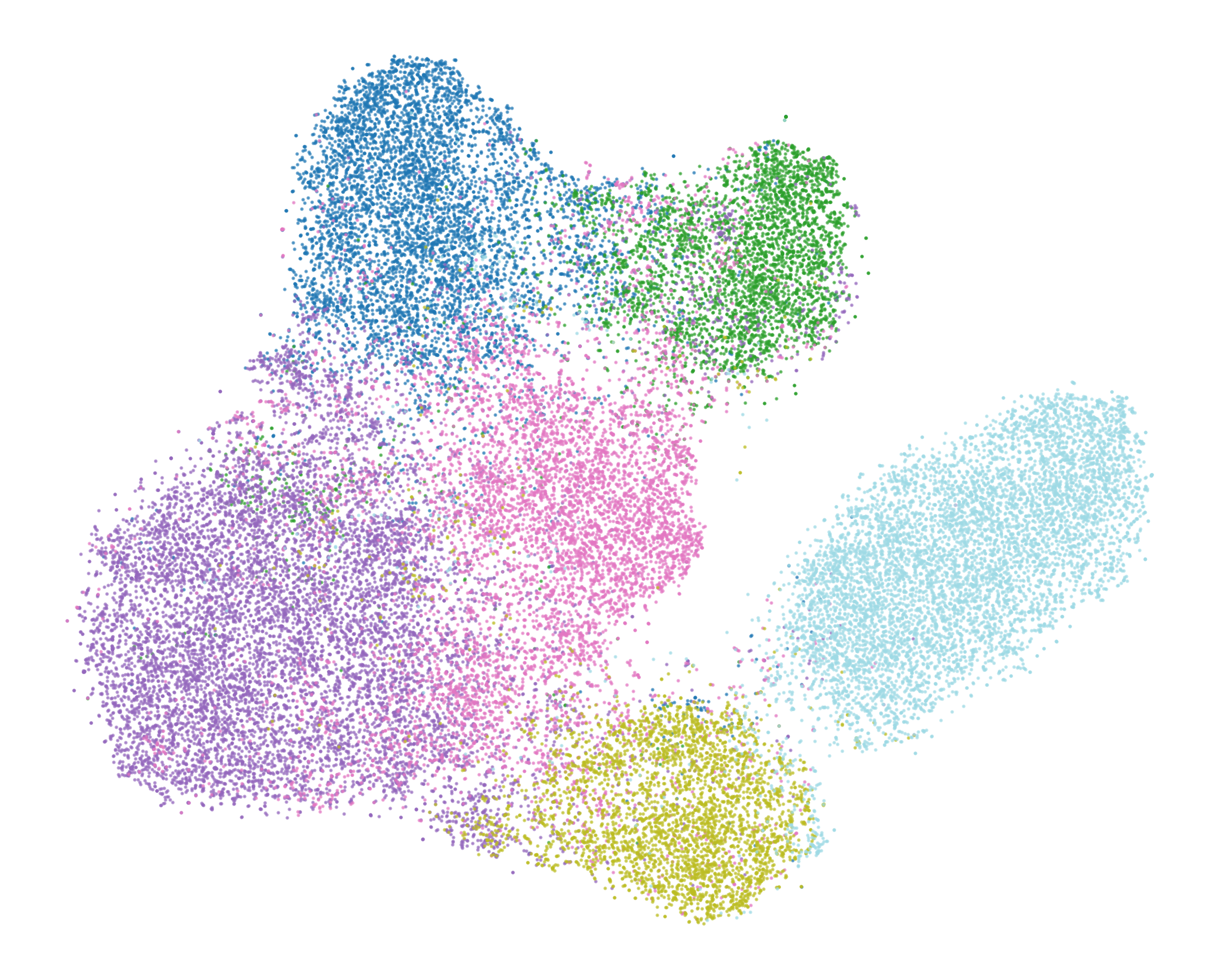}}
\subfigure[DBLP(t=2)]
{
\includegraphics[width=4.5cm,height =3.6cm]
{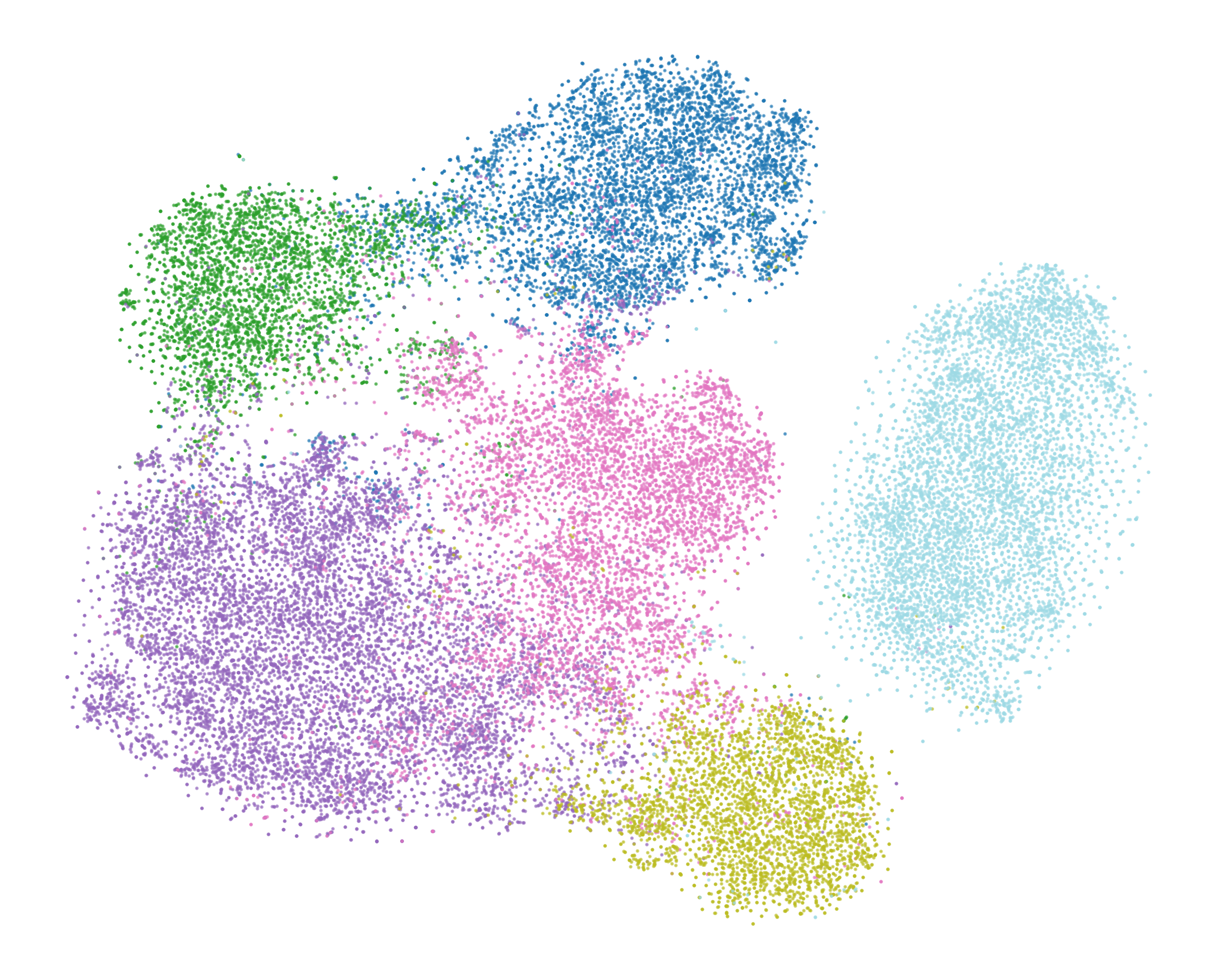}}
\caption{(a)(b)(c) show the visualization of Pubmed node features at $t = 0$, $t = 1$ and $t = 2$. (d)(e)(f) show the visualization of DBLP node features at $t = 0$, $t = 1$ and $t = 2$. }
\label{vis}
\end{figure*}
\subsection{Robustness analysis}
To assess the robustness of RFHND under noisy input conditions, extensive experiments are conducted considering both feature-level and structure-level perturbations. The experimental configurations are as follows:
\begin{itemize}
\item \textbf{Feature Noise:}
Node features in the Citeseer dataset are perturbed using three types of noise:
(a) Gaussian noise sampled from $N(0,\sigma^2)$ and scaled according to the noise rate;
(b) Uniform noise drawn from $Unif(-\delta,\delta)$, where $\delta$ is proportional to the noise rate;
(c) Mask noise, in which a portion of feature values is randomly set to zero based on the noise rate.
\item \textbf{Structure Noise:}
For the Cora-CA dataset, the hypergraph structure is disturbed by randomly deleting existing hyperedges and introducing  hyperedges composed of random node subsets, with the overall perturbation controlled by the noise rate.
\end{itemize}

As illustrated in Fig. \ref{noise}, the proposed approach consistently surpasses baseline methods under all noise settings. When exposed to feature-level perturbations, RFHND demonstrates strong robustness on the Citeseer dataset, maintaining stable performance across different noise types, including Gaussian, uniform, and mask-based disturbances. Notably, all examined models show higher tolerance to mask noise compared with other feature perturbations, possibly because mask noise introduces structured sparsity rather than random distortion, allowing models to better adapt during training.

The analysis of structural noise demonstrates comparable benefits on the Cora-CA dataset. Across all examined noise levels (0.1–0.4), the proposed method consistently delivers better performance. While the accuracy of all models decreases as the noise rate grows, RFHND shows only a marginal decline, reflecting its strong robustness against structural disturbances. 

\subsection{Parameter analysis}
To assess the sensitivity of RFHND to key hyperparameters, an extensive  analysis was performed focusing on two crucial factors: the hidden layer dimension and the step size $\tau$. In particular, the model’s performance was evaluated under different configurations of hidden layer dimensions $\left\{16, 32, 64, 128, 256\right\}$ and step sizes $\left\{0.2, 0.5, 1, 2, 4\right\}$.

Specifically, we selected three representative datasets: Zoo, NTU2012 and ModelNet40.  As illustrated in Fig. \ref{para}, RFHND remains relatively stable under different hidden layers. However, when the step size becomes excessively large, the model performance deteriorates, indicating that overly large steps may cause unstable feature propagation and weaken the model’s capacity to capture fine-grained relational patterns. This is consistent with the conclusion of Theorem \ref{thm:st}.

\subsection{Feature visualization}
To gain deeper insights into the dynamic changes of node representations, a feature visualization analysis was conducted on the Citeseer and Pubmed datasets. In this analysis, node features were captured at multiple time points during the ODE integration process, specifically at $t = 0$, $t = 1$, and $t = 2$. The obtained high-dimensional representations were then mapped onto a two-dimensional space using a dimensionality reduction method for visualization purposes. In the resulting visualizations, nodes are color-coded according to their corresponding class labels.

As illustrated in Fig. \ref{vis}, it can be seen that the node embeddings gradually form clearer and more compact clusters as time advances. This phenomenon suggests that the model progressively enhances its ability to distinguish nodes from different classes, with the representations evolving into more separable and discriminative structures throughout the integration process.

\section{Conclusion}
In this paper, we mitigated the critical challenge of over-smoothing in hypergraph neural networks by generalizing Ricci flow from differential geometry to the hypergraph domain. Specifically, we proposed RFHND (Ricci Flow-guided Hypergraph Neural Diffusion), a novel message passing paradigm that leverages discrete Ricci flow to adaptively regulate information propagation. By formulating node feature evolution as a continuous PDE-based system, our method effectively controls the diffusion rate at the geometric level, thereby preventing feature homogenization. Furthermore, we provided a rigorous theoretical analysis validating the convergence properties and approximation capabilities of our framework. Extensive experiments across multiple benchmark datasets demonstrate that RFHND significantly outperforms state-of-the-art methods and exhibits strong robustness. Ultimately, this work offers a promising geometric perspective for addressing over-smoothing, enriching the theory  for high-order representation learning.

\bibliographystyle{IEEEtran}
\bibliography{IEEEabrv,refer}

\begin{thebibliography}{10}
\providecommand{\url}[1]{#1}
\csname url@samestyle\endcsname
\providecommand{\newblock}{\relax}
\providecommand{\bibinfo}[2]{#2}
\providecommand{\BIBentrySTDinterwordspacing}{\spaceskip=0pt\relax}
\providecommand{\BIBentryALTinterwordstretchfactor}{4}
\providecommand{\BIBentryALTinterwordspacing}{\spaceskip=\fontdimen2\font plus
\BIBentryALTinterwordstretchfactor\fontdimen3\font minus \fontdimen4\font\relax}
\providecommand{\BIBforeignlanguage}[2]{{%
\expandafter\ifx\csname l@#1\endcsname\relax
\typeout{** WARNING: IEEEtran.bst: No hyphenation pattern has been}%
\typeout{** loaded for the language `#1'. Using the pattern for}%
\typeout{** the default language instead.}%
\else
\language=\csname l@#1\endcsname
\fi
#2}}
\providecommand{\BIBdecl}{\relax}
\BIBdecl

\bibitem{zhou2006learning}
D.~Zhou, J.~Huang, and B.~Sch{\"o}lkopf, ``Learning with hypergraphs: Clustering, classification, and embedding,'' \emph{Advances in neural information processing systems}, vol.~19, 2006.

\bibitem{antelmi2023survey}
A.~Antelmi, G.~Cordasco, M.~Polato, V.~Scarano, C.~Spagnuolo, and D.~Yang, ``A survey on hypergraph representation learning,'' \emph{ACM Computing Surveys}, vol.~56, no.~1, pp. 1--38, 2023.

\bibitem{berge1984hypergraphs}
C.~Berge, \emph{Hypergraphs: combinatorics of finite sets}.\hskip 1em plus 0.5em minus 0.4em\relax Elsevier, 1984, vol.~45.

\bibitem{zlatic2009hypergraph}
V.~Zlati{\'c}, G.~Ghoshal, and G.~Caldarelli, ``Hypergraph topological quantities for tagged social networks,'' \emph{Physical Review E—Statistical, Nonlinear, and Soft Matter Physics}, vol.~80, no.~3, p. 036118, 2009.

\bibitem{zhu2018social}
J.~Zhu, J.~Zhu, S.~Ghosh, W.~Wu, and J.~Yuan, ``Social influence maximization in hypergraph in social networks,'' \emph{IEEE Transactions on Network Science and Engineering}, vol.~6, no.~4, pp. 801--811, 2018.

\bibitem{xia2022self}
L.~Xia, C.~Huang, and C.~Zhang, ``Self-supervised hypergraph transformer for recommender systems,'' in \emph{Proceedings of the 28th ACM SIGKDD conference on knowledge discovery and data mining}, 2022, pp. 2100--2109.

\bibitem{wang2020next}
J.~Wang, K.~Ding, L.~Hong, H.~Liu, and J.~Caverlee, ``Next-item recommendation with sequential hypergraphs,'' in \emph{Proceedings of the 43rd international ACM SIGIR conference on research and development in information retrieval}, 2020, pp. 1101--1110.

\bibitem{feng2021hypergraph}
S.~Feng, E.~Heath, B.~Jefferson, C.~Joslyn, H.~Kvinge, H.~D. Mitchell, B.~Praggastis, A.~J. Eisfeld, A.~C. Sims, L.~B. Thackray \emph{et~al.}, ``Hypergraph models of biological networks to identify genes critical to pathogenic viral response,'' \emph{BMC bioinformatics}, vol.~22, no.~1, p. 287, 2021.

\bibitem{klamt2009hypergraphs}
S.~Klamt, U.-U. Haus, and F.~Theis, ``Hypergraphs and cellular networks,'' \emph{PLoS computational biology}, vol.~5, no.~5, p. e1000385, 2009.

\bibitem{kim2024survey}
S.~Kim, S.~Y. Lee, Y.~Gao, A.~Antelmi, M.~Polato, and K.~Shin, ``A survey on hypergraph neural networks: an in-depth and step-by-step guide,'' in \emph{Proceedings of the 30th ACM SIGKDD Conference on Knowledge Discovery and Data Mining}, 2024, pp. 6534--6544.

\bibitem{wu2022hypergraph}
H.~Wu and M.~K. Ng, ``Hypergraph convolution on nodes-hyperedges network for semi-supervised node classification,'' \emph{ACM Transactions on Knowledge Discovery from Data (TKDD)}, vol.~16, no.~4, pp. 1--19, 2022.

\bibitem{yadati2020nhp}
N.~Yadati, V.~Nitin, M.~Nimishakavi, P.~Yadav, A.~Louis, and P.~Talukdar, ``Nhp: Neural hypergraph link prediction,'' in \emph{Proceedings of the 29th ACM international conference on information \& knowledge management}, 2020, pp. 1705--1714.

\bibitem{li2013link}
D.~Li, Z.~Xu, S.~Li, and X.~Sun, ``Link prediction in social networks based on hypergraph,'' in \emph{Proceedings of the 22nd international conference on world wide web}, 2013, pp. 41--42.

\bibitem{li2025deep}
M.~Li, Y.~Fang, Y.~Wang, H.~Feng, Y.~Gu, L.~Bai, and P.~Lio, ``Deep hypergraph neural networks with tight framelets,'' in \emph{Proceedings of the AAAI Conference on Artificial Intelligence}, vol.~39, no.~17, 2025, pp. 18\,385--18\,392.

\bibitem{chen2022preventing}
G.~Chen, J.~Zhang, X.~Xiao, and Y.~Li, ``Preventing over-smoothing for hypergraph neural networks,'' \emph{arXiv preprint arXiv:2203.17159}, 2022.

\bibitem{lin2022deephgnn}
J.~LIN, Z.~YE, H.~ZHAO, and L.~FANG, ``Deephgnn: A novel deep hypergraph neural network,'' \emph{Chinese Journal of Electronics}, vol.~31, no.~5, pp. 958--968, 2022.

\bibitem{huang2021unignn}
J.~Huang and J.~Yang, ``Unignn: a unified framework for graph and hypergraph neural networks,'' \emph{arXiv preprint arXiv:2105.00956}, 2021.

\bibitem{chien2021you}
E.~Chien, C.~Pan, J.~Peng, and O.~Milenkovic, ``You are allset: A multiset function framework for hypergraph neural networks,'' \emph{arXiv preprint arXiv:2106.13264}, 2021.

\bibitem{bai2021hypergraph}
S.~Bai, F.~Zhang, and P.~H. Torr, ``Hypergraph convolution and hypergraph attention,'' \emph{Pattern Recognition}, vol. 110, p. 107637, 2021.

\bibitem{chen2020hypergraph}
C.~Chen, Z.~Cheng, Z.~Li, and M.~Wang, ``Hypergraph attention networks,'' in \emph{2020 IEEE 19th International Conference on Trust, Security and Privacy in Computing and Communications (TrustCom)}.\hskip 1em plus 0.5em minus 0.4em\relax IEEE, 2020, pp. 1560--1565.

\bibitem{nguyen2023revisiting}
K.~Nguyen, N.~M. Hieu, V.~D. Nguyen, N.~Ho, S.~Osher, and T.~M. Nguyen, ``Revisiting over-smoothing and over-squashing using ollivier-ricci curvature,'' in \emph{International Conference on Machine Learning}.\hskip 1em plus 0.5em minus 0.4em\relax PMLR, 2023, pp. 25\,956--25\,979.

\bibitem{chow2004ricci}
B.~Chow and D.~Knopf, \emph{The ricci flow: An introduction: An introduction}.\hskip 1em plus 0.5em minus 0.4em\relax American Mathematical Soc., 2004, vol.~1.

\bibitem{yang2009generalized}
Y.-L. Yang, R.~Guo, F.~Luo, S.-M. Hu, and X.~Gu, ``Generalized discrete ricci flow,'' in \emph{Computer graphics forum}, vol.~28, no.~7.\hskip 1em plus 0.5em minus 0.4em\relax Wiley Online Library, 2009, pp. 2005--2014.

\bibitem{feng2019hypergraph}
Y.~Feng, H.~You, Z.~Zhang, R.~Ji, and Y.~Gao, ``Hypergraph neural networks,'' in \emph{Proceedings of the AAAI conference on artificial intelligence}, vol.~33, no.~01, 2019, pp. 3558--3565.

\bibitem{dong2020hnhn}
Y.~Dong, W.~Sawin, and Y.~Bengio, ``Hnhn: Hypergraph networks with hyperedge neurons,'' \emph{arXiv preprint arXiv:2006.12278}, 2020.

\bibitem{yadati2019hypergcn}
N.~Yadati, M.~Nimishakavi, P.~Yadav, V.~Nitin, A.~Louis, and P.~Talukdar, ``Hypergcn: A new method for training graph convolutional networks on hypergraphs,'' \emph{Advances in neural information processing systems}, vol.~32, 2019.

\bibitem{arya2020hypersage}
D.~Arya, D.~K. Gupta, S.~Rudinac, and M.~Worring, ``Hypersage: Generalizing inductive representation learning on hypergraphs,'' \emph{arXiv preprint arXiv:2010.04558}, 2020.

\bibitem{yan2024hypergraph}
J.~Yan, Y.~Feng, S.~Ying, and Y.~Gao, ``Hypergraph dynamic system,'' in \emph{The twelfth international conference on learning representations}, 2024.

\bibitem{xie2025k}
L.~Xie, S.~Gao, J.~Liu, M.~Yin, and T.~Jin, ``K-hop hypergraph neural network: A comprehensive aggregation approach,'' in \emph{Proceedings of the AAAI Conference on Artificial Intelligence}, vol.~39, no.~20, 2025, pp. 21\,679--21\,687.

\bibitem{wang2022equivariant}
P.~Wang, S.~Yang, Y.~Liu, Z.~Wang, and P.~Li, ``Equivariant hypergraph diffusion neural operators,'' \emph{arXiv preprint arXiv:2207.06680}, 2022.

\bibitem{leal2021forman}
W.~Leal, G.~Restrepo, P.~F. Stadler, and J.~Jost, ``Forman--ricci curvature for hypergraphs,'' \emph{Advances in Complex Systems}, vol.~24, no.~01, p. 2150003, 2021.

\bibitem{coupette2022ollivier}
C.~Coupette, S.~Dalleiger, and B.~Rieck, ``Ollivier-ricci curvature for hypergraphs: A unified framework,'' \emph{arXiv preprint arXiv:2210.12048}, 2022.

\bibitem{eidi2020edge}
M.~Eidi, A.~Farzam, W.~Leal, A.~Samal, and J.~Jost, ``Edge-based analysis of networks: curvatures of graphs and hypergraphs,'' \emph{Theory in biosciences}, vol. 139, no.~4, pp. 337--348, 2020.

\bibitem{cai2020note}
C.~Cai and Y.~Wang, ``A note on over-smoothing for graph neural networks,'' \emph{arXiv preprint arXiv:2006.13318}, 2020.

\bibitem{hamilton1982three}
R.~S. Hamilton, ``Three-manifolds with positive ricci curvature,'' \emph{Journal of Differential geometry}, vol.~17, no.~2, pp. 255--306, 1982.

\bibitem{ollivier2009ricci}
Y.~Ollivier, ``Ricci curvature of markov chains on metric spaces,'' \emph{Journal of Functional Analysis}, vol. 256, no.~3, pp. 810--864, 2009.

\bibitem{ni2019community}
C.-C. Ni, Y.-Y. Lin, F.~Luo, and J.~Gao, ``Community detection on networks with ricci flow,'' \emph{Scientific reports}, vol.~9, no.~1, p. 9984, 2019.

\bibitem{ni2018network}
C.-C. Ni, Y.-Y. Lin, J.~Gao, and X.~Gu, ``Network alignment by discrete ollivier-ricci flow,'' in \emph{International symposium on graph drawing and network visualization}.\hskip 1em plus 0.5em minus 0.4em\relax Springer, 2018, pp. 447--462.

\bibitem{chen2025graph}
J.~Chen, B.~Deng, C.~Chen, Z.~Zheng \emph{et~al.}, ``Graph neural ricci flow: Evolving feature from a curvature perspective,'' in \emph{The Thirteenth International Conference on Learning Representations}, 2025.

\bibitem{cybenko1989approximation}
G.~Cybenko, ``Approximation by superpositions of a sigmoidal function,'' \emph{Mathematics of control, signals and systems}, vol.~2, no.~4, pp. 303--314, 1989.

\bibitem{dua2017uci}
D.~Dua, C.~Graff \emph{et~al.}, ``Uci machine learning repository, 2017,'' \emph{URL http://archive. ics. uci. edu/ml}, vol.~7, no.~1, p.~62, 2017.

\bibitem{wu20153d}
Z.~Wu, S.~Song, A.~Khosla, F.~Yu, L.~Zhang, X.~Tang, and J.~Xiao, ``3d shapenets: A deep representation for volumetric shapes,'' in \emph{Proceedings of the IEEE conference on computer vision and pattern recognition}, 2015, pp. 1912--1920.

\bibitem{chen2003visual}
D.-Y. Chen, X.-P. Tian, Y.-T. Shen, and M.~Ouhyoung, ``On visual similarity based 3d model retrieval,'' in \emph{Computer graphics forum}.\hskip 1em plus 0.5em minus 0.4em\relax Wiley Online Library, 2003, pp. 223--232.

\bibitem{amburg2020clustering}
I.~Amburg, N.~Veldt, and A.~Benson, ``Clustering in graphs and hypergraphs with categorical edge labels,'' in \emph{Proceedings of the web conference 2020}, 2020, pp. 706--717.

\bibitem{chodrow2021generative}
P.~S. Chodrow, N.~Veldt, and A.~R. Benson, ``Generative hypergraph clustering: From blockmodels to modularity,'' \emph{Science Advances}, vol.~7, no.~28, p. eabh1303, 2021.

\bibitem{fowler2006connecting}
J.~H. Fowler, ``Connecting the congress: A study of cosponsorship networks,'' \emph{Political analysis}, vol.~14, no.~4, pp. 456--487, 2006.

\bibitem{wang2019heterogeneous}
X.~Wang, H.~Ji, C.~Shi, B.~Wang, Y.~Ye, P.~Cui, and P.~S. Yu, ``Heterogeneous graph attention network,'' in \emph{The world wide web conference}, 2019, pp. 2022--2032.

\bibitem{fey2019fast}
M.~Fey and J.~E. Lenssen, ``Fast graph representation learning with pytorch geometric,'' \emph{arXiv preprint arXiv:1903.02428}, 2019.

\bibitem{chen2018neural}
R.~T. Chen, Y.~Rubanova, J.~Bettencourt, and D.~K. Duvenaud, ``Neural ordinary differential equations,'' \emph{Advances in neural information processing systems}, vol.~31, 2018.

\bibitem{deshpande2018contextual}
Y.~Deshpande, S.~Sen, A.~Montanari, and E.~Mossel, ``Contextual stochastic block models,'' \emph{Advances in Neural Information Processing Systems}, vol.~31, 2018.

\bibitem{ghoshdastidar2014consistency}
D.~Ghoshdastidar and A.~Dukkipati, ``Consistency of spectral partitioning of uniform hypergraphs under planted partition model,'' \emph{Advances in Neural Information Processing Systems}, vol.~27, 2014.

\bibitem{chien2018community}
I.~Chien, C.-Y. Lin, and I.-H. Wang, ``Community detection in hypergraphs: Optimal statistical limit and efficient algorithms,'' in \emph{International conference on artificial intelligence and statistics}.\hskip 1em plus 0.5em minus 0.4em\relax PMLR, 2018, pp. 871--879.

\end{thebibliography}
\clearpage
\setcounter{page}{1}
\appendices
\section{HYPEREDGE CURVATURE ON WEIGHTED HYPERGRAPHS}
\label{sec:hyperedge_curvature}
Discrete curvature is a key tool for analyzing the geometric properties of networks. In recent years, multiple studies have successfully generalized this concept from graph data to more complex hypergraphs. This section will focus on elaborating two definitions of curvature defined on hyperedges.
\subsection{Forman-Ricci Curvature on Hypergraphs}
\label{fror}
Forman-Ricci (FR) curvature~\cite{leal2021forman}, first introduced on cell complexes, has now been extended to hypergraphs. The curvature characterizes the balance between a hyperedge's internal structure and its external connections, essentially measuring whether it functions as a community core (high curvature) or a bridge (low curvature). A common definition for the Forman-Ricci curvature of a hyperedge $e \in \mathcal{E}$ is given by:
\begin{equation}
    \kappa^{FR}_{e} = \sum_{i\in e}w_{i}-w_{e}\sum_{k\in e,e_{l}\sim k}\sqrt{\frac{w_{k}}{w_{e_{l}}}},
    \label{eq:forman_hypergraph}
\end{equation}
where $w_e$ is the weight of the hyperedge $e$ and $w_i$ is the weight of the node $i$. $e_{l}\sim k$ represent the represents the hyperedge associated with node k. In this paper, we consider the weights of all nodes as 1. Then, the curvature expression mentioned above can be written in the following form:
\begin{equation}
    \kappa^{FR}_{e} = |e|-\sum_{k\in e,e_{l}\sim k}{\frac{w_e}{\sqrt{w_{e_{l}}}}}.
\label{eq:forman_hypergraph_after}
\end{equation}

\subsection{Ollivier-Ricci Curvature on Hypergraphs}
Ollivier-Ricci(OR) curvature~\cite{coupette2022ollivier}, first introduced using optimal transport theory, has now been extended to hypergraphs. The curvature characterizes network robustness by measuring the transport cost between local neighborhood distributions; a low cost implies a robust, community-like structure (high curvature), while a high cost indicates a fragile bridge (low curvature). 

To define OR curvature on a hypergraph, we first need to define a probability measure $\mu_i$ for each node $i \in V$. This typically represents a two-step random walk from the node. The standard definition is:
\begin{equation}
    \mu_i(j) =\sum_{e\supseteq \{i,j\} } \left( \frac{w_e}{\sum_{i \ni f} w_f} \cdot \frac{1}{|e|} \right),
    \label{eq:random_walk}
\end{equation}
where $j$ is a neighbor of $i$ , and $w_{e}'$ is the weight of the hyperedge $e$. 
The OR curvature of the hyperedge $e$ is then defined as:
\begin{equation}
    \kappa^{OR}_e = 1 - \frac{2}{|e|(|e|-1)}\sum_{i,j\in e}\frac{W_1(\mu_i, \mu_j)}{d(i, j)},
    \label{eq:orc_pair}
\end{equation}
where $W_1(\mu_i, \mu_j)$ is the 1-Wasserstein distance  between the two probability distributions, and $d(i, j)$ is the shortest path distance between nodes $i$ and $j$ in the hypergraph.

\section{PROOF OF THEOREM 1}
\label{appendix:proofdi}
\begin{Lemma}
\label{lemmak}
 Consider the attribute discrete Ricci flow over the interval $[t_1, t_2]$. If the hyperedge weight $w_{e}(t)$ has a finite number of $N$ extrema points within $(t_1, t_2)$, denoted by $T_1, \dots, T_N$, and we define $T_0 = t_1$ and $T_{N+1} = t_2$, then:
\begin{equation}
\frac{\rho_e - 1}{\zeta_e} \mathbb{E}_{t \in [t_1, t_2]}( |\kappa_{e}(t)| ) > \frac{1 - \rho_e^{-1}}{t_2 - t_1},
\end{equation}
where $\zeta_e = \min_{0 \le i \le N} (T_{i+1} - T_i)$ represents the length of the minimum monotonic interval within $[t_1, t_2]$, and $\rho_e = \frac{w_e^{max}}{w_e^{min}}$ denotes the ratio of the maximum to the minimum value of $w_e(t)$. Specifically, if $w_{e}(t)$ is monotonic within $(t_1, t_2)$, we have:
\begin{equation}
\begin{split}
&\frac{1}{\min\{w_e(t_2), w_e(t_1)\}} \frac{|w_e(t_2) - w_e(t_1)|}{t_2 - t_1} \\&> \mathbb{E}_{t \in [t_1, t_2]}( |\kappa_{e}(t)| ),
\end{split}
\end{equation}
this shows that $|w_e(t_2) -w_e(t_1)| \to 0$ implies $\mathbb{E}_{t \in [t_1, t_2]}(|\kappa_{e}(t)| ) \to 0$.
\end{Lemma}
\begin{proof}
 We integrate Equation (\ref{ricci}) over an arbitrary time interval $[t_1, t_2]$:
\begin{equation}
w_e(t_2) - w_e(t_1) = - \int_{t_1}^{t_2} \kappa_e(t) w_e(t) dt. 
\end{equation}
Suppose that within $[t_1, t_2]$, the function $w_e(t)$ has $N$ extremum points. These $N$ points divide the interval $[t_1, t_2]$ into $N+1$ adjacent monotonic sub-intervals. Given that $w_e(t) > 0$,  the sign of $\frac{\partial w_e(t)}{\partial t}$ is determined by $\kappa_e(t)$. Moreover, the $\kappa_e(t)$ sign remains non-negative or non-positive within each monotonic sub-interval.

Let $[T_i, T_{i+1}]$ denote any one of these $N+1$ monotonic sub-intervals. If $\kappa_e(t) \ge 0$ for $t \in [T_i, T_{i+1}]$, then $w_e(t)$ is monotonically decreasing, which implies $w_e(T_i) > w_e(T_{i+1}) > 0$. Therefore:
\begin{equation}
\begin{split}
&0 > -w_e(T_{i+1}) \int_{T_i}^{T_{i+1}} \kappa_e(t) dt > -\int_{T_i}^{T_{i+1}} \kappa_e(t) w_e(t) dt \\&> -w_e(T_i) \int_{T_i}^{T_{i+1}} \kappa_e(t) dt. 
\end{split}
\end{equation}
If $\kappa_e(t) \le 0$ for $t \in [T_i, T_{i+1}]$, then $w_e(t)$ is monotonically increasing, which implies $w_e(T_{i+1}) > w_e(T_i) > 0$ and
\begin{equation}
\begin{split}
&-w_e(T_{i+1}) \int_{T_i}^{T_{i+1}} \kappa_e(t) dt >- \int_{T_i}^{T_{i+1}} \kappa_e(t) w_e(t) dt \\&> -w_e(T_i) \int_{T_i}^{T_{i+1}} \kappa_e(t) dt > 0. 
\end{split}
\end{equation}
Therefore, on any monotonic sub-interval, we have:
\begin{equation}
\begin{split}
&-w_e(T_{i+1}) \int_{T_i}^{T_{i+1}} \kappa_e(t) dt > w_e(T_{i+1}) - w_e(T_i) \\&> -w_e(T_{i}) \int_{T_i}^{T_{i+1}} \kappa_e(t) dt. 
\end{split}
\end{equation}
Take the absolute value of the above inequality:

\begin{equation}
\begin{split}
&\max\{w_e(T_i), w_e(T_{i+1})\} \int_{T_i}^{T_{i+1}} |\kappa_e(t)| dt\\
&>|w_e(T_{i+1})- w_e(T_i)| \\
&> \min\{w_e(T_i), w_e(T_{i+1})\} \int_{T_i}^{T_{i+1}}|\kappa_e(t)|dt.
\end{split}
\end{equation}
Let \begin{equation}
\begin{split}
&w_e^{\max}>\max\{w_e(T_i), w_e(T_{i+1})\}\\&>\min\{w_e(T_i), w_e(T_{i+1})\}>w_e^{\min},
\end{split}
\end{equation} we have:
\begin{equation}
\begin{split}
&\frac{|w_e(T_{i+1}) - w_e(T_i)|}{w_e^{\min}} > \int_{T_i}^{T_{i+1}} |\kappa_e(t)| dt \\&>\frac{|w_e(T_{i+1}) - w_e(T_i)|}{w_e^{\max}}. 
\end{split}
\end{equation}
Summing over $i$ from $0$ to $N$ and then dividing by $t_2 - t_1$:
\begin{equation}
\begin{split}
&\frac{\sum_{i=0}^{N} |w_e(T_{i+1}) - w_e(T_i)|}{w_e^{\min}(t_2 - t_1)} > \frac{1}{t_2 - t_1} \int_{t_1}^{t_2} |\kappa_e(t)| dt \\&= \mathbb{E}_{t \in [t_1, t_2]}( |\kappa_e(t)| ) > \frac{\sum_{i=0}^{N} |w_e(T_{i+1}) - w_e(T_i)|}{w_e^{\max}(t_2 - t_1)}. 
\end{split}
\end{equation}
Because $(N+1)(w_e^{\max} - w_e^{\min}) > \sum_{i=0}^{N} |w_e(T_{i+1}) - w_e(T_i)| > w_e^{\max} - w_e^{\min}$, and $\rho_e = \frac{w_e^{\max}}{w_e^{\min}}$:
\begin{equation}
\frac{(N+1)(\rho_e-1)}{t_2-t_1} > \mathbb{E}_{t \in [t_1, t_2]}(|\kappa_{e}(t)| ) > \frac{1-\rho_e^{-1}}{t_2-t_1}. 
\end{equation}
According to the definition of $\zeta_e$, we have $(t_2-t_1) > \zeta_e(N+1)$. Thus, the estimation of the upper bound can be replaced by $\zeta_e^{-1}(\rho_e-1) > \frac{(N+1)(\rho_e-1)}{t_2-t_1}$. Thus,
\begin{equation}
\frac{\rho_e - 1}{\zeta_e} \mathbb{E}_{t \in [t_1, t_2]}( |\kappa_{e}(t)|) > \frac{1 - \rho_e^{-1}}{t_2 - t_1}.
\end{equation}

\end{proof}

\begin{Theorem}
\label{thedi}
Consider the attribute discrete Ricci flow with hyperedge weight as equation (\ref{weight}), and $|\mathbf{x}(t)|\equiv1$. If for any e in $\mathcal{H}$, $w_{e}(t)$ has a finite number of monotonic on $[t_{1},t_{2}]$, then the average Dirichlet energy within $[t_{1},t_{2}]$ has following bound:
\begin{equation}
B_{1} \geq \mathbb{E}_{t\in[t_{1},t_{2}]}(E(\mathbf{X}(t))\geq B_{2} ,
\end{equation}
where \begin{equation}
\begin{split}
B_1=(t_2 - t_1) \sum_{e \in \mathcal{E}} \left(c_{e} \rho_{e} \zeta_{e}^{-1}\right) -  \sum_{e \in \mathcal{E}} \left(\alpha_{e}\epsilon \rho_{e}\right),\\
B_2=\sum_{e\in \mathcal{E}}\left(c_e-\left(2+\epsilon\right)\alpha_{e}\zeta_{e}^{-1}(t_{2}-t_{1})\right).
\end{split}
\end{equation}
Here, $c_{e}=\sum_{i,j \in e}\left(\frac{1}{2|e|}\left(\frac{1}{d_{i}}+\frac{1}{d_{j}}\right)+\left(1+\epsilon\right)\alpha_{e}\right)$, $d_i$ denotes the degree of node $i$, $\rho_e$ denotes the ratio of the maximum to the minimum value of $w_e(t)$ and $\zeta_e$ represents the length of the
minimum monotonic interval of $w_e(t)$ within $[t_1,t_2]$.
\end{Theorem}
\begin{proof}
As equation (\ref{de}) shows \begin{equation}
\begin{split}
    E(X(t)) &= \frac{1}{2}\sum_{e\in \mathcal{E}}\sum_{i,j \in e}\frac{1}{|e|}\left( \frac{\mathbf{x}_{i}(t)}{\sqrt{d_{i}}}-\frac{\mathbf{x}_{j}(t)}{\sqrt{d_{j}}}\right)^{2}
    \\
&=\frac{1}{2}\sum_{e\in \mathcal{E}}\sum_{i,j \in e}\frac{1}{|e|}\left(\frac{1}{d_{i}}+\frac{1}{d_{j}}-2\frac{{\mathbf{x}_{i}(t)}^{T}{\mathbf{x}_{j}(t)}}{\sqrt{d_{i}d_{j}}}\right) 
\\
&=\sum_{e\in \mathcal{E}}\left(\sum_{i,j \in e}\frac{1}{2|e|}\left(\frac{1}{d_{i}}+\frac{1}{d_{j}}\right)+\left(1+\epsilon\right)\alpha_{e}\right)\\&-\sum_{e\in \mathcal{E}}\alpha_{e}w_{e}(t).
\end{split}
\end{equation}
Here,we define $c_{e}=\sum_{i,j \in e}\frac{1}{2|e|}\left(\frac{1}{d_{i}}+\frac{1}{d_{j}}\right)+\left(1+\epsilon\right)\alpha_{e}$ and $E_{e}(t)=c_{e}-\alpha_{e}w_{e}(t),$ then the Dirichlet energy can  be recorded as:
\begin{equation}
E(\mathbf{X}(t))=\sum_{e\in \mathcal{E}}E_{e} = \sum_{e\in \mathcal{E}}\left(c_{e}-\alpha_{e}w_{e}(t)\right).
\end{equation}
Take the partial derivative of $t$ and substitute the equation into Ricci flow:
\begin{equation}
\begin{split}
\frac{\partial E(\mathbf{X}(t))}{\partial t}& = \sum_{e\in \mathcal{E}}\left(-\alpha_{e}\frac{\partial w_{e}(t)}{\partial t}\right)
\\
&=\sum_{e\in \mathcal{E}}\left(-\alpha_{e}\left({-\kappa_{e}(t) w_{e}(t)}\right)\right)
\\
&=\sum_{e\in \mathcal{E}}\kappa_{e}(t)\left(c_{e}-E_{e}(t)\right).
\end{split}
\end{equation}
By transferring the above results to a hyperedge, we can obtain the following formula:
\begin{equation}
\frac{\partial E_{e}(\mathbf{X}(t))}{\partial t} =\kappa_{e}(t)(c_{e}-E_{e}(t)).
\end{equation}
Integrating the above formula within the interval $[t_{1},t_{2}]$ :
\begin{equation}
E_{e}(t_{2})-E_{e}(t_{1})=c_{e}\int_{t_{1}}^{t_{2}}  \kappa_{e}(t) dt-\int_{t_{1}}^{t_{2}}  \kappa_{e}(t)E_{e}(t)dt.
\end{equation}
Then,we consider the N extremum points of $w_{e}(t)$ within $(t_{1},t_{2}):T_{1},...,T_{N}.$ For any $i$, $k_{e}(t)$ does not change its sign within $(T_{i},T_{i+1})$, and $E_{e}\geq0$, so we have $|\int_{T_{i}}^{T_{i+1}}k_{e}(t)E_{e}(t)dt|=\int_{T_{i}}^{T_{i+1}}|k_{e}(t)|E_{e}(t)dt$, so:
\begin{equation}
\begin{split}
    &|E_{e}(T_{i+1})-E_{e}(T_{i})|\\
    &=|c_{e}\int_{T_{i}}^{T_{i+1}}  \kappa_{e}(t) dt-\int_{T_{i}}^{T_{i+1}}  \kappa_{e}(t)E_{e}(t)dt|\\
    &=|c_{e}\int_{T_{i}}^{T_{i+1}}  \kappa_{e}(t) dt|-|\int_{T_{i}}^{T_{i+1}}  \kappa_{e}(t)E_{e}(t)dt| \\
     &=c_{e}\int_{T_{i}}^{T_{i+1}}  |\kappa_{e}(t) |dt-\int_{T_{i}}^{T_{i+1}} |\kappa_{e}(t)|E_{e}(t)dt \\
     &\approx c_{e}\int_{T_{i}}^{T_{i+1}}  |\kappa_{e}(t) |dt\\  &-\mathbb{E}_{t\in[t_{1},t_{2}]}|\kappa_{e}(t)|\int_{T_{i}}^{T_{i+1}} E_{e}(t)dt .
\label{13}
\end{split}
\end{equation}
Sum over $i$ from 0 to N and then divide by $t_{2}-t_{1}$:
\begin{equation}
\begin{split}
&\frac{1}{t_{2}-t_{1}}\sum _{i=0}^N|E_{e}(T_{i+1})-E_{e}(T_{i})|\\ &\geq \mathbb{E}_{t\in[t_{1},t_{2}]}|\kappa_{e}(t)|\left(c_{e}-\mathbb{E}_{t\in[t_{1},t_{2}]}(E_{e}(t))\right).
\end{split}
\end{equation}
Let $\zeta_{e}= min_{0\leq i \leq N}(T_{i+1}-T_{i})$ represents the length of the minimum monotonic interval within $[t_{1},t_{2}]$, so $(N+1)\zeta_{e} \leq t_{2}-t_{1}$ and $w_{e}(t)\leq 2+\epsilon$, then we have:
\begin{equation}
    \begin{split}
       &\frac{1}{t_{2}-t_{1}}\sum _{i=0}^N|E_{e}(T_{i+1})-E_{e}(T_{i})|\\&= \frac{1}{t_{2}-t_{1}}\sum_{i=0}^{N}|\alpha_{e}w_{e}(T_{i+1})-\alpha_{e}w_{e}(T_{i})|\\
       &\leq  \frac{\alpha_{e}(N+1)(\max w_{e}-\min w_{e})}{t_{2}-t_{1}}\\
       &\leq \zeta_{e}^{-1} \alpha_{e} \max w_{e}(1-\rho_{e}^{-1})\\
       &\leq (2+\epsilon)\zeta_{e}^{-1} \alpha_{e} (1-\rho_{e}^{-1}).
    \end{split}
\end{equation}
Here, $\rho_{e}$ denotes the ratio of the maximum to the minimum
value of $w_{e}(t)$. Thus, we have:
\begin{equation}
\begin{split}
    &(2+\epsilon)\zeta_{e}^{-1} \alpha_{e} (1-\rho_{e}^{-1})\\&\geq \mathbb{E}_{t\in[t_{1},t_{2}]}|\kappa_{e}(t)|\left(c_{e}-\mathbb{E}_{t\in[t_{1},t_{2}]}(E_{e}(t)) \right).
\end{split}
\end{equation}
Then we can obtain the lower bound for $\mathbb{E}_{t\in[t_{1},t_{2}]}(E_{e}(t))$:
\begin{equation}
    \begin{split}
        \mathbb{E}_{t\in[t_{1},t_{2}]}(E_{e}(t))&\geq c_{e}-\frac{(2+\epsilon)\zeta_{e}^{-1}(1-\rho_{e}^{-1})\alpha_{e}}{\mathbb{E}_{t\in[t_{1},t_{2}]}(|k_{e}(t)|)}\\
        &\geq c_e-\left(2+\epsilon\right)\alpha_{e}\zeta_{e}^{-1}(t_{2}-t_{1}).
    \end{split}
\end{equation}
Take the sum over all edge:
\begin{equation}
\begin{split}
\label{low}
 \mathbb{ E}_{t\in[t_{1},t_{2}]}(E(t))&\geq
   \sum_{e\in \mathcal{E}}\left(c_e-\left(2+\epsilon\right)\alpha_{e}\zeta_{e}^{-1}(t_{2}-t_{1})\right).
\end{split}
\end{equation}
Now let's derive the upper bound of the energy value.  From equation (\ref{13}), we can obtained:
\begin{equation}
\begin{split}
&\int_{T_{i}}^{T_{i+1}} |\kappa_{e}(t)|E_{e}(t)dt\\
& =c_{e}\int_{T_{i}}^{T_{i+1}}  |\kappa_{e}(t) |dt-|E_{e}(T_{i+1})-E_{e}(T_{i})|.
\end{split}
\end{equation}
Sum over $i$ from 0 to $N$:
\begin{equation}
\begin{split}
& \int_{t_{1}}^{t_{2}} |\kappa_{e}(t)|E_{e}(t)dt \\
& =c_{e}\int_{t_{1}}^{t_{2}}  |\kappa_{e}(t) |dt-\sum_{i=0}^{N}|E_{e}(T_{i+1})-E_{e}(T_{i}))|\\
& =c_{e}\int_{t_{1}}^{t_{2}}  |\kappa_{e}(t) |dt-\sum_{i=0}^{N}\alpha_{e}|w_{e}(T_{i+1})-w_{e}(T_{i})|\\
& \leq c_{e}\int_{t_{1}}^{t_{2}}  |\kappa_{e}(t) |dt-\alpha_{e}|\max w_{e}- \min w_{e}|\\
& \leq c_{e}\int_{t_{1}}^{t_{2}}  |\kappa_{e}(t) |dt-\alpha_{e}\epsilon(\rho_{e}-1).
\end{split}
\end{equation}
Divide both sides by  $t_2 - t_1$:
\begin{equation}
\begin{split}
&\mathbb{E}_{t \in [t_1, t_2]}\left(|\kappa_{e}(t)| E_{e}(t)\right) 
\\&\leq c_{e} \mathbb{E}_{t \in [t_1, t_2]}(|\kappa_{e}(t)|) - \frac{\alpha_{e}\epsilon (\rho_{e}^{-1} - 1)}{(t_2 - t_1) }.
\end{split}
\end{equation}
Applying Lemma \ref{lemmak}, we obtain:
\begin{equation}
\begin{split}
&\frac{1 - \rho_{e}^{-1}}{t_2 - t_1} \mathbb{E}_{t \in [t_1, t_2]}(E_{e}(t)) \\&\leq c_{e} \zeta_{e}^{-1} (\rho_{e} - 1) - \frac{\alpha_{e}\epsilon (\rho_{e} - 1)}{(t_2 - t_1) }.
\end{split}
\end{equation}

Due to $E(t) = \sum_{e\in \mathcal{E}} E_{e}(t)$, we can derive the upper bound as:
\begin{equation}
\label{up}
\begin{split}
\mathbb{E}_{t \in [t_1, t_2]}(E(t)) &\leq (t_2 - t_1) \sum_{e \in \mathcal{E}} \left(c_{e} \rho_{e} \zeta_{e}^{-1}\right) -  \sum_{e \in \mathcal{E}} \left(\alpha_{e}\epsilon \rho_{e}\right). 
\end{split}
\end{equation}
Combining the results of Equation (\ref{low}) and Equation (\ref{up}), we have:
\begin{equation}
    \begin{split}
     &(t_2 - t_1) \sum_{e \in \mathcal{E}} \left(c_{e} \rho_{e} \zeta_{e}^{-1}\right) -  \sum_{e \in \mathcal{E}} \left(\alpha_{e}\epsilon \rho_{e}\right)\\ &\geq
     \mathbb{E}_{t \in [t_1, t_2]}(E(t)) \\ &\geq
   \sum_{e\in \mathcal{E}}\left(c_e-\left(2+\epsilon\right)\alpha_{e}\zeta_{e}^{-1}(t_{2}-t_{1})\right).\\
    \end{split}
\end{equation}
\end{proof}
Based on the above Theorem \ref{thedi}, we can draw the conclusion of Theorem \ref{thedir}, The specific proof is as follows:
\begin{proof}
    When $H$ is a non-regular hypergraph and for all $e$ in $\mathcal{H}$,  $w_{e}(t)$  are monotonic on $[t_{1},t_{2}]$. So $\zeta_{e} =t_2-t_1.$ The upper bound of the Theorem \ref{thedi} can be written as:
    \begin{equation}
    \begin{split}
     &(t_2 - t_1) \sum_{e \in \mathcal{E}} \left(c_{e} \rho_{e} \zeta_{e}^{-1}\right) -  \sum_{e \in \mathcal{E}} \left(\alpha_{e}\epsilon \rho_{e}\right)\\&=(t_2 - t_1) \sum_{e \in \mathcal{E}} \rho_{e} \zeta_{e}^{-1}\left( \sum_{i,j \in e}\frac{1}{2|e|}\left(\frac{1}{d_{i}}+\frac{1}{d_{j}}\right)+\left(1+\epsilon\right)\alpha_{e} \right)\\& -  \sum_{e \in \mathcal{E}} \left(\alpha_{e}\epsilon \rho_{e}\right)\\&=(t_2 - t_1) \sum_{e \in \mathcal{E}} \rho_{e} \zeta_{e}^{-1}\left( \sum_{i,j \in e}\frac{1}{2|e|}\left(\frac{1}{d_{i}}+\frac{1}{d_{j}}\right) \right)\\& + \sum_{e \in \mathcal{E}}\left((t_2 - t_1)\rho_{e} \zeta_{e}^{-1} \left(1+\epsilon\right)- \epsilon \rho_{e}\right)\alpha_{e}
     \\&=\sum_{e \in \mathcal{E}} \rho_{e} \left( \sum_{i,j \in e}\frac{1}{2|e|}\left(\frac{1}{d_{i}}+\frac{1}{d_{j}}\right) \right)\ +  \sum_{e \in \mathcal{E}}\rho_{e} \alpha_{e}
     \\& \leq \rho_{max}\left(\sum_{e\in \mathcal{E}}\sum_{i\in e, j\in e}\frac{1}{2|e|}\left(\frac{1}{d_i} +\frac{1}{d_j}\right)+\sum_{e\in \mathcal{E}}\alpha_e \right). 
    \end{split}
\end{equation}
Similarly, the lower bound can be written as:
\begin{equation}
    \begin{split}
   &\sum_{e\in \mathcal{E}}\left(c_e-\left(2+\epsilon\right)\alpha_{e}\zeta_{e}^{-1}(t_{2}-t_{1})\right)\\&=\sum_{e\in \mathcal{E}}\left(\sum_{i,j \in e}\frac{1}{2|e|}\left(\frac{1}{d_{i}}+\frac{1}{d_{j}}\right)+\left(1+\epsilon\right)\alpha_{e}\right)\\&-\sum_{e\in \mathcal{E}}\left(2+\epsilon\right)\alpha_{e}\zeta_{e}^{-1}(t_{2}-t_{1})\\&=
   \sum_{e \in \mathcal{E}}  \left( \sum_{i,j \in e}\frac{1}{2|e|}\left(\frac{1}{d_{i}}+\frac{1}{d_{j}}\right) \right)\ -  \sum_{e \in \mathcal{E}} \alpha_{e}.
    \end{split}
\end{equation}
From Equation (\ref{weight}),we have
$\alpha_{e} = \frac{1}{|e|}\sum_{i\in e, j\in e}\frac{1}{\sqrt{d_{i}d_{j}}}$,
So\begin{equation}
\begin{split}
&\sum_{i,j \in e}\frac{1}{2|e|}\left(\frac{1}{d_{i}}+\frac{1}{d_{j}}\right) \ - \alpha_e\\&=\sum_{i,j \in e}\left(
     \frac{1}{2|e|}\left(\frac{1}{d_{i}}+\frac{1}{d_{j}}\right) \ - \frac{1}{|e|} \frac{1}{\sqrt{d_{i}d_{j}}}\right)\geq 0.
 \end{split}
\end{equation}
Because $\mathcal{H}$ is a non-regular hypergraph, so the lower bound is greater than 0.
\end{proof}
\section{PROOF OF Proposition \ref{proposition3}}
\begin{proof}
From the condition $|\mathbf{x}_i(t)| \equiv 1$, we know that $\partial_t \|\mathbf{x}_i(t)\|^2 = 0$, which expands using the chain rule to:
\begin{equation}
2 \left\langle \mathbf{x}_i(t), \frac{\partial \mathbf{x}_i(t)}{\partial t} \right\rangle = 0.
\end{equation}

Moreover, due to $w_{e}(t)\equiv \frac{1}{\alpha_{e} } [\frac{1}{|e|}\sum_{i\in e, j\in e}\frac{cos(\mathbf{x}_{i}(t),\mathbf{x}_{j}(t))}{\sqrt{d_{i}d_{j}}}]+1+\epsilon$, we know that, for any $e$ that $i\in e$:
\begin{equation} 
\frac{\partial w_e(t)}{\partial \mathbf{x}_i} = \frac{1}{\sum_{i,j\in e}\frac{1}{\sqrt{d_id_j}}}\left(\sum_{j\in e}\frac{\mathbf{x}_j(t)}{\sqrt{d_id_j}}\right).
\end{equation} 
We  denote $\frac{\partial \mathbf{x}_i(t)}{\partial t}$ as $\mathbf{y}$. Let $\mu_{ie} = -\frac{\kappa_{e}(t) w_e(t)}{1 + \sum_{j\in e, j\neq i}\lambda^{j}(\mathbf{x}_1(t),... \mathbf{x}_{|e|}(t))}$. Then, the original problem takes the following form:
\begin{align}
\min \quad & \mathbf{y}^T \mathbf{y} \nonumber \\
\text{s.t.} \quad & \frac{1}{\sum_{i,j\in e}\frac{1}{\sqrt{d_id_j}}}\left(\sum_{j\in e}\frac{\mathbf{x}_j(t)^T}{\sqrt{d_id_j}}\right)\mathbf{y} = \mu_{ie}, \\
& \mathbf{x}_i(t)^T \mathbf{y} = 0.
\label{min}
\end{align}

we denote
\begin{equation}
m_{ie}=\frac{1}{\sum_{i,j\in e}\frac{1}{\sqrt{d_id_j}}}\left(\sum_{j\in e}\frac{\mathbf{x}_j(t)}{\sqrt{d_id_j}}\right).\end{equation}
Then consider the method of Lagrange multipliers:
\begin{align}
\label{LA}
\nabla_y L(\mathbf{y}) &= 0 \nonumber \\
\rightarrow \quad \nabla_y (\mathbf{y}^T \mathbf{y} + \lambda(m_{ie}^T \mathbf{y} - \mu_{ie}) + \tau \mathbf{x}_i(t)^T \mathbf{y}) &= 0 \nonumber \\
\rightarrow \quad 2\mathbf{y} + \lambda m_{ie} + \tau \mathbf{x}_i(t) &= 0.
\end{align}

Left-multiplying Equation (\ref{LA}) by $\mathbf{x}_i(t)^T$, and combining it with Equation (\ref{LA}) and $\mathbf{x}_i(t)^T \mathbf{x}_i(t) = 1$, we get:
\begin{equation}
\tau = -\lambda \mathbf{x}_i(t)^T m_{ie}.
\end{equation}

Similarly, left-multiplying Equation (\ref{min}) by $m_{ie}^T$:
\begin{align}
&2m_{ie}^T \mathbf{y} + \lambda + \tau m_{ie}^T \mathbf{x}_i(t) = 0 \nonumber \\
&\rightarrow \quad 2\mu_{ie} + \lambda - \lambda(\mathbf{x}_i(t)^T m_{ie}) = 0 \nonumber \\
&\rightarrow \quad \lambda = -2\mu_{ie} \left( 1 - (\mathbf{x}_i(t)^T m_{ie})^2 \right)^{-1}.
\end{align}

Thus we can obtain the solution to this optimization problem:
\begin{equation}
\begin{split}
\mathbf{y}^* &= -\frac{\lambda m_{ie} + \tau \mathbf{x}_i(t)}{2} \\&= -\frac{\lambda  m_{ie}  - \lambda \mathbf{x}_i(t)^T m_{ie} \mathbf{x}_i(t)}{2}\\& = -\frac{\lambda(\mathbf{I} - \mathbf{x}_i(t) \mathbf{x}_i(t)^T) m_{ie} }{2}  \\
&= \frac{\mu_{ie}(\mathbf{I} - \mathbf{x}_i(t) \mathbf{x}_i(t)^T)m_{ie} }{1 - (\mathbf{x}_i(t)^T m_{ie} )^2}\\& = \frac{\mu_{ie} \left[ m_{ie}  - \cos(\mathbf{x}_i(t), m_{ie} ) \mathbf{x}_i(t) \right]}{1 - (\mathbf{x}_i(t)^T m_{ie} )^2} .
\end{split}
\end{equation}
So,\begin{equation}
\begin{split}
\mathbf{y}^* 
&=  \frac{\mu_{ie}}{\mathbf{1} -{ (\mathbf{x}_i(t)^T m_{ie})}^{2}}\left[ m_{ie}  - \cos(\mathbf{x}_i(t), m_{ie} ) \mathbf{x}_i(t) \right]\\
&= \frac{\mu_{ie}}{\mathbf{1} -{ (\mathbf{x}_i(t)^Tm_{ie} )}^{2}}\\&\left[ \frac{1}{\sum_{i,j\in e}\frac{1}{\sqrt{d_id_j}}}\left(\sum_{i,j\in e}\frac{\mathbf{x}_j(t)}{\sqrt{d_id_j}}\right)  - \cos(\mathbf{x}_i(t), m_{ie} ) \mathbf{x}_i(t) \right]\\
&= \frac{\mu_{ie}}{(\mathbf{1} -{ (\mathbf{x}_i(t)^Tm_{ie} )}^{2})(\sum_{i,j\in e}\frac{1}{\sqrt{d_id_j}})}\\&\left[ \left(\sum_{j\in e}\frac{\mathbf{x}_j(t)}{\sqrt{d_id_j}}\right)  - \sum_{j\in e}\frac{\cos(\mathbf{x}_i(t), m_{ie} ) \mathbf{x}_i(t)}{\sqrt{d_id_j}} \right]\\
&=-\kappa'_{ie}(t) \left[\sum_{j\in e} \frac{\mathbf{x}_j(t) - \cos \left( \mathbf{x}_i(t),m_{ie} \right) \mathbf{x}_i(t)}{\sqrt{d_id_j}} \right].
\end{split}
\end{equation}
Here, $\kappa'_{ie}(t)=\frac{\mu_{ie}}{(\mathbf{1} -{ (\mathbf{x}_i(t)^Tm_{ie} )}^{2})(\sum_{i,j\in e}\frac{1}{\sqrt{d_id_j}})}$.  The $\|\mathbf{y}\| \ge 0$ always holds. Therefore, $\mathbf{y}^*$ corresponds to the point of minimum $\mathbf{y}$.
\end{proof}

\section{proof of Theorem \ref{exit}}
\begin{Lemma}
\label{LEMMA9}
$C(\mathbf{X}(t))$ is locally Lipschitz continuous.
\end{Lemma}
\begin{proof}
Given that  $\|\mathbf{x}_i\| = 1$, the gradient norm of $\cos(\mathbf{x}_i, \mathbf{x}_j)$ with respect to $(\mathbf{x}_i, \mathbf{x}_j)$ is bounded. Thus, there exists a constant $L_C$ and bounded sets such that
\begin{equation}
\|C(\mathbf{X}(t)) - C(\mathbf{Y}(t))\|_F \leq L_C \|\mathbf{X}(t) - \mathbf{Y}(t)\|_F. 
\end{equation}
So, $C(\mathbf{X}(t))$ is locally Lipschitz continuous.
\end{proof}








\begin{Lemma}
\label{LEMMA10}
\label{lem:proof_properties}
Under the assumption of Theorem \ref{exit}, the functions $m_{ie}$, $\mu_{ie}$, and $\kappa_{ie}^\prime(t)$ are bounded and Lipschitz continuous.
\end{Lemma}

\begin{proof}
    We prove the boundedness and Lipschitz continuity for each function based on the assumptions of Theorem \ref{exit}.

    {{1. $m_{ie}$ is bounded and Lipschitz continuous:}}
    By the definition of $m_{ie}$, it can be written as a convex combination of $x_j$:
   \begin{equation}
    m_{ie} = \sum_{j} \alpha_j \mathbf{x}_j, \quad \text{where } \alpha_j \geq 0 \text{ and } \sum_{j} \alpha_j = 1.
    \end{equation}
    From this convex combination, we immediately obtain the Lipschitz property:
    \begin{equation}
    \|m_{ie}(\mathbf{X}(t)) - m_{ie}(\mathbf{Y}(t))\|_F \leq \|\mathbf{X}(t) - \mathbf{Y}(t)\|_F.
    \end{equation}
    Furthermore, since the vector norm is bounded by unity, $m_{ie}$ is also bounded.

    {2. $\mu_{ie}$ is bounded and Lipschitz continuous:}
    The function $\mu_{ie}$ is defined as:
     \begin{equation}
   \mu_{ie} = -\frac{\kappa_{e}(t) w_e(t)}{1 + \sum_{j\in e, j\neq i}\lambda^{j}(\mathbf{x}_1(t),... \mathbf{x}_{|e|}(t))}.
     \end{equation}
    Since $\kappa_e(t)$ and $w_e(t)$ are bounded, and ${\lambda^j}(\cdot)$ is bounded and locally Lipschitz, $\mu_{ie}$ is clearly bounded on bounded sets and locally Lipschitz.

    {3. $\kappa_{ie}^\prime(t)$ is bounded and Lipschitz continuous:}
    The function $\kappa_{ie}^\prime(t)$ is given by:
    \begin{align}
    \kappa'_{ie}(t) &= \frac{\mu_{ie}}{(1 - (\mathbf{x}_i(t)^T m_{ie} )^2) S_e}, \\
    \text{where } S_e &= \sum_{p, q \in e} \frac{1}{\sqrt{d_p d_q}} \in \left[\frac{|e|^2}{d_{\max}}, \frac{|e|^2}{d_{\min}}\right].
    \end{align}
    From the result in part 2, $\mu_{ie}$ is bounded and locally Lipschitz. Assumption  guarantees that the denominator, $(1 - (\mathbf{x}_i(t)^T m_{ie})^2) S_e$ is bounded below by the strictly positive quantity $\varepsilon S_e$, where $\varepsilon > 0$. Therefore, $\kappa'_{ie}(t)$ is the ratio of two bounded, locally Lipschitz functions whose denominator is bounded away from zero, implying $\kappa'_{ie}(t)$ is also bounded on bounded sets and locally Lipschitz.
\end{proof}

\begin{Lemma}
\label{LEMMA11}
The  function $S(\mathbf{X}(t))$ is locally Lipschitz continuous and bounded.
\end{Lemma}

\begin{proof}
The matrix $S(\mathbf{X}(t)) = D_e^{-\frac{1}{2}} H K'(\mathbf{X}(t)) H^T D_e^{-\frac{1}{2}}$ is linear with respect to $K'(\mathbf{X}(t))$. By Lemma \ref{lem:proof_properties} , each diagonal element  of $K'(\mathbf{X}(t))$ is locally Lipschitz. Therefore, $S(\mathbf{X}(t))$ is locally Lipschitz. Furthermore, $S(\mathbf{X}(t))$ is bounded:
\begin{equation}
\|S(\mathbf{X}(t))\| \leq c_H \max_{e} |\kappa'_e(\mathbf{X}( t))| \mathrel{\mathop:}= M(t) .
\end{equation}
where $c_H $ is a  constant independent of time $t$.
\end{proof}

Based on the above Lemma, we can draw the conclusion of Theorem \ref{exit}, The specific proof is as follows:
\begin{proof}
Let \begin{equation}
F(\mathbf{X}(t)) =   \operatorname{diag}\Big( \big( S(\mathbf{X}(t)) \odot C(\mathbf{X}(t)) \big) \mathbf{1}_{N} \Big) - S(\mathbf{X}(t)) ,\end{equation} 
\begin{equation}G(\mathbf{X}(t)) =  \bigg[ \operatorname{diag}\Big( \big( S(\mathbf{X}(t)) \odot C(\mathbf{X}(t)) \big) \mathbf{1}_{N} \Big) - S(\mathbf{X}(t)) \bigg] \mathbf{X}(t).\end{equation}  

\noindent
Combining Lemmas \ref{LEMMA9} and \ref{LEMMA11} with the bound $|C(\mathbf{X}(t))| \leq 1$, it follows that there exists a constant $L_H(t)$ such that the following inequality holds on the bounded set:
\begin{equation}
\label{eq:F_properties}
\|F(\mathbf{X}(t)) - F(\mathbf{Y}(t))\| \leq L_H(t) \|\mathbf{X}(t) - \mathbf{Y}(t)\|, 
\end{equation}
\begin{equation}
\|F(\mathbf{X}(t))\| \leq 2 \|S(\mathbf{X}(t))\| \leq 2M(t). 
\end{equation}
Then, the Lipschitz continuity of $G(\mathbf{X}(t))$ is established as follows:
\begin{equation}
\begin{split}
\label{eq:G_lip_proof}
\|G&(\mathbf{X}(t)) - G(\mathbf{Y}(t))\| \\
=& \|F(\mathbf{X}(t)) \mathbf{X}(t) - F(\mathbf{Y}(t))\mathbf{Y}(t)\|  \\
\leq& \|F(\mathbf{X}(t))\mathbf{X}(t) - F(\mathbf{Y}(t)) \mathbf{X}(t)\| \\
&+ \|F(\mathbf{Y}(t)) \mathbf{X}(t) - F(\mathbf{Y}(t)) \mathbf{Y}(t)\|  \\
\leq& \|F(\mathbf{X}(t)) - F(\mathbf{Y}(t))\| \|\mathbf{X}(t)\| + \|F(\mathbf{Y}(t))\| \|\mathbf{X}(t) - \mathbf{Y}(t)\| \\
\leq& (2M(t) + L_H(t) \|\mathbf{Y}(t)\|) \|\mathbf{X}(t) - \mathbf{Y}(t)\|. 
\end{split}
\end{equation}
So, $G(\mathbf{X}(t))$ is locally  Lipschitz
continuous. Furthermore, the linear growth bound is given by:
\begin{equation}
\label{grow}
\|G(\mathbf{X}(t))\| \leq \|F(\mathbf{X}(t))\| \|\mathbf{X}(t)\| \leq 2M(t) \|\mathbf{X}(t)\|. 
\end{equation}
By the Picard-Lindelöf theorem, for any initial value $\mathbf{X}(0) $, there exists a unique solution on a maximal time interval $[0, \tau)$ (where $\tau > 0$).
From the growth bound (\ref{grow}) and Grönwall's inequality, we obtain:
\begin{equation}
\|\mathbf{X}(t)\| \leq \|\mathbf{X}(0)\| \exp\left(2 \int_0^t M(s)\, ds\right). 
\end{equation}
Since $M(\cdot)$ is integrable on $[0, T]$, the solution remains bounded on the interval. This prevents finite-time blow-up, thereby ensuring that the solution can be uniquely extended to the full interval $[0, T]$.
\end{proof}

\section{proof of Theorem \ref{thm:st}}
\label{app:e}
\begin{proof}
Let \begin{equation}
F(\mathbf{X}(t)) =   \operatorname{diag}\Big( \big( S(\mathbf{X}(t)) \odot C(\mathbf{X}(t)) \big) \mathbf{1}_{N} \Big) - S(\mathbf{X}(t)) ,\end{equation} and the 
explicit Euler method updates the state ${\mathbf{X}}{(t+1)}$ based on the current state ${\mathbf{X}}{(t)}$. The discrete form of the update is given by:
\begin{equation}
    {\mathbf{X}}{(t+1)} = {\mathbf{X}}{(t)} - \tau \cdot F(\mathbf{X}(t)) {\mathbf{X}}{(t)},
\end{equation}
where $\tau$ is the time step.

Next, we expand the squared norm of the updated state:
\begin{equation}
\begin{split}
    &\|{\mathbf{X}}{(t+1)}\|^2 \\&= \left( {\mathbf{X}}{(t)} - \tau F(\mathbf{X}(t)) {\mathbf{X}}{(t)} \right)^\top \left( {\mathbf{X}}{(t)} - \tau F(\mathbf{X}(t)) {\mathbf{X}}{(t)} \right).
\end{split}
\end{equation}
This can be simplified to:
\begin{equation}
\begin{split}
    &\|{\mathbf{X}}{(t+1)}\|^2 \\&= \|{\mathbf{X}}{(t)}\|^2 - 2\tau  {\mathbf{X}}{(t)}^T F(\mathbf{X}(t)) {\mathbf{X}}{(t)}  + \tau^2 \|F(\mathbf{X}(t)) {\mathbf{X}}{(t)}\|^2.
\end{split}
\end{equation}

For stability, we require that the norm of the updated state is not greater than the norm of the previous state, i.e.,
\begin{equation}
    \|{\mathbf{X}}{(t+1)}\|^2 \leq \|{\mathbf{X}}{(t)}\|^2.
\end{equation}
This implies:
\begin{equation}
    - 2\tau  {\mathbf{X}}{(t)}^T F(\mathbf{X}(t)){\mathbf{X}}{(t)}+ \tau^2 \|F(\mathbf{X}(t)) {\mathbf{X}}{(t)}\|^2 \leq 0.
\end{equation}
We can  rewrite this expression as:
\begin{equation}
\label{buchang}
    \tau \leq \frac{2  {\mathbf{X}}{(t)}^T F(\mathbf{X}(t)){\mathbf{X}}{(t)}}{\|F(\mathbf{X}(t)) {\mathbf{X}}{(t)}\|^2} = \frac{2 {\mathbf{X}}{(t})^\top F(\mathbf{X}(t)) {\mathbf{X}}{(t)}}{{\mathbf{X}}{(t)}^\top F(\mathbf{X}(t))^2 {\mathbf{X}}{(t)}}.
\end{equation}
The right side of the inequality is the Rayleigh quotient for the matrix $F(\mathbf{X}(t))$.

Now, We denote the largest eigenvalue of $F(\mathbf{X}(t))$ by $\lambda_{\max}(F(\mathbf{X}(t)))$.
Since $F(\mathbf{X}(t))$ is symmetric, its spectral radius satisfies $\lambda_{\max}(F(\mathbf{X}(t))) = \|F(\mathbf{X}(t))\|_2$. By the triangle inequality, we can obtain:
\begin{equation}
\begin{aligned}
    &\|F(\mathbf{X}(t))\|_2 \\&\leq \|\operatorname{diag}(S({\mathbf{X}}{(t)}) \odot C({\mathbf{X}}{(t)}))\mathbf{1}_{N}\|_2 + \|S({\mathbf{X}}{(t)})\|_2 \\
    &= \|(S({\mathbf{X}}{(t)}) \odot C({\mathbf{X}}{(t)}))\mathbf{1}_{N}\|_\infty + \|S({\mathbf{X}}{(t)})\|_2.
\end{aligned}
\end{equation}

Moreover, since $|c_{ij}| \leq 1$,
\begin{equation}
\begin{aligned}
    \|(S(\mathbf{X}(t)) \odot C(\mathbf{X}(t)))\mathbf{1}_{N}\|_\infty 
    &\leq \max_{i} \sum_{j} |s_{ij} c_{ij}| \\
    &\leq \max_{i} \sum_{j} s_{ij}.
\end{aligned}
\end{equation}

Similarly, since $S({\mathbf{X}}{(t)})$ is symmetric, we have
\begin{equation}
    \|S({\mathbf{X}}{(t)})\|_2 \leq \|S({\mathbf{X}}{(t)})\|_\infty = \max_{i} \sum_{j} s_{ij}.
\end{equation}
Combining the inequalities  above yields the following bound:
\begin{equation}
    \lambda_{max}(F(\mathbf{X}(t))) = \|F(\mathbf{X}(t))\|_2 \leq 2 \max_{i} \sum_{j} s_{ij}.
\end{equation}
According to the Equation (\ref{buchang}), the stability condition requires $\tau \leq \frac{2}{\lambda_{\max}(F(\mathbf{X}(t)))}$. Substituting the upper bound derived above, we conclude that a sufficient condition for stability is:
\begin{equation}
\tau \leq \frac{1}{\max_{i} \sum_{j} s_{ij}}.
\end{equation}
Thus, the explicit Euler method for the RFHND is stable provided that this condition is met.


\end{proof}

\section{Analysis of computational complexity}
\label{Aocc}
In this section, we conduct a detailed analysis of the complexity of RFHND. Let $n$, $m$, and $N=\sum_{e \in \mathcal{E}} |e|$ denote the number of nodes, hyperedges, and incidence pairs, respectively.  $d_{\text{in}}$ and $d_{\text{out}}$ represent the dimensions of the input and output features, $T$ indicates the number of layers.

The per-step computational cost consists of three main components: 
\begin{itemize}
    \item 
Sparse scatter operations and matrix multiplication $S(\mathbf{X}(t))\mathbf{X}(t)$, scaling with $O(Nd_{\text{in}})$. \item Edge-wise MLP transformations, scaling with $O(md^2_{\text{in}})$.  \item  Cosine similarity computations. We restrict similarity calculations to node pairs induced by hyperedges, resulting in $O(M_2 d_{\text{in}})$ where $M_2 = \sum_{e \in \mathcal{E}} |e|^2$. 
\end{itemize}
Consequently, the total time complexity for $T$ steps is $O(T(Nd_{\text{in}} + M_2 d_{\text{in}} + md^2_{\text{in}}))$. 
When the hyperedge size is bouded ($|e| \le r$),  then $N = O(mr)$ and $M_2 = O(mr^2)$. This ensures the model scales linearly with the number of hyperedges. Adding the complexity of the linear transformations in the input and output layers, the total complexity is
\begin{equation}
 O(n d_{\text{in}} d_{\text{out}}) + T \cdot O(m r^2 d_{\text{in}} + md^2_{\text{in}}) .
\end{equation}

\end{document}